\documentclass{article}

\usepackage{microtype}
\usepackage{graphicx}
\usepackage{subcaption}
\usepackage{booktabs}
\usepackage{hyperref}

\usepackage[accepted]{icml2026}

\usepackage{amsmath}
\usepackage{amssymb}
\usepackage{mathtools}
\usepackage{amsthm}
\usepackage{wrapfig}

\usepackage{array}
\usepackage{xurl}
\newcolumntype{P}[1]{>{\centering\arraybackslash}p{#1}}

\usepackage[capitalize,noabbrev]{cleveref}

\newcommand{\RR}{\mathbb{R}}
\newcommand{\NN}{\mathbb{N}}
\newcommand{\cA}{\mathcal{A}}
\newcommand{\cS}{\mathcal{S}}
\newcommand{\cO}{\mathcal{O}}
\newcommand{\cP}{\mathcal{P}}
\newcommand{\cW}{\mathcal{W}}

\newcommand{\ourmethod}{\textsc{LIFT}}
\newcommand{\cw}[1]{\pi_{\mathrm{cw}, #1}}

\newcommand{\eps}{\epsilon}

\theoremstyle{plain}
\newtheorem{theorem}{Theorem}[section]
\newtheorem{prop}[theorem]{Proposition}
\newtheorem{lemma}[theorem]{Lemma}
\newtheorem{cor}[theorem]{Corollary}
\newtheorem{thm}[theorem]{Theorem}
\theoremstyle{definition}
\newtheorem{defn}[theorem]{Definition}

\theoremstyle{remark}

\usepackage[textsize=tiny]{todonotes}

\icmltitlerunning{Trajectory-Level Data Augmentation for Offline Reinforcement Learning}

\begin{document}

\usetikzlibrary{decorations.pathreplacing, decorations.pathmorphing}
\usetikzlibrary{shapes.misc}

\twocolumn[
  \icmltitle{Trajectory-Level Data Augmentation for Offline Reinforcement Learning}
  \icmlsetsymbol{equal}{*}

  \begin{icmlauthorlist}
    \icmlauthor{Tobias Schmähling}{yyy}
    \icmlauthor{Matthias Burkhardt}{yyy}
    \icmlauthor{Tobias Windisch}{yyy}

  \end{icmlauthorlist}

  \icmlaffiliation{yyy}{University of Applied Sciences Kempten, Kempten, Germany}

  \icmlcorrespondingauthor{Tobias Schmähling}{tobias.schmaehling@hs-kempten.de}
  \icmlkeywords{Machine Learning, ICML}

  \vskip 0.3in
]

\printAffiliationsAndNotice{} 

\begin{abstract}
We propose a data augmentation method for offline reinforcement learning, motivated by active positioning problems.
Particularly, our approach enables the training of off-policy models from a limited number of
suboptimal trajectories. We introduce a trajectory-based augmentation technique that exploits task structure and the
geometric relationship between rewards, value functions, and mathematical properties of logging policies. During data
collection, our augmentation supports suboptimal logging policies, leading to higher data quality
and improved offline reinforcement learning performance.
We provide theoretical justification for these strategies and validate them empirically across positioning tasks of varying dimensionality and under partial observability.
\end{abstract}

\section{Introduction}

Offline reinforcement learning promises to learn effective decision-making policies
from static, pre-collected datasets, avoiding the cost and risk of online
exploration~\citep{offline-rl-survey}. This is particularly attractive in real systems, where
trial-and-error interaction is expensive or unsafe. Yet the central challenge of
offline RL is equally well known: because learning is constrained to the support of the dataset,
distribution shift between the learned policy and the data-generating behavior can lead to severe
extrapolation errors and brittle, suboptimal performance. 
Contemporary methods therefore rely on conservative updates, regularizing toward the behavior distribution 
or warm-starting from the logging policy before attempting improvement. However, these algorithmic safeguards do 
not remove the core dependency on the data itself.
Consequently, offline RL performance can depend strongly on the
quality of the logging policy that produced the data. Prior evidence shows that
dataset selection can outweigh algorithmic differences~\citep{dataset-perspective, d4rl, exorl},
suggesting that the logging policy effectively sets the attainable frontier for an offline learner.
While the field has developed a rich set of algorithms for coping with imperfect data (see Section~\ref{s:related_work}),
actionable principles for improving the data-generating process remain scarce. 

We observe an algorithmic gap in what can be done with a given logging policy to improve RL, with
pure offline learning on the one end and offline-to-online fine-tuning on the other. Motivated by
this, we seek to understand ways in the middle, particularly how logging policies can be augmented
already \emph{during collection} in a principled way to generate better data for RL. While prior 
work has studied the effect of exploration ~\cite{zhang2023active}, we focus on exploiting
logged data to improve learning. Here, a key practical obstacle to improve datasets that is easily
overlooked is the \emph{hand-off problem}. In many applications, the logging policy is not a stochastic, exploratory controller,
but a deterministic, scripted process with internal state. Injecting a better action mid-trajectory can invalidate its assumptions, forcing a
restart before execution can safely resume.

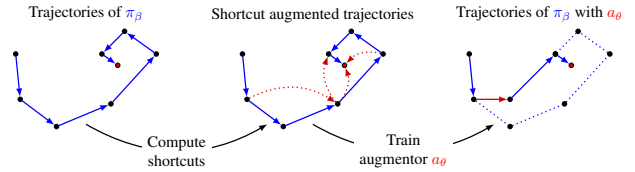
\begin{figure}[ht]
    \centering
\scalebox{0.6}{
    \begin{tikzpicture}

    \node[] (initial-collection) at (0,0) {
        \begin{tikzpicture}
        \node[draw, inner sep=1pt, circle, fill=black] at (1,0) (s1) {};
        \node[draw, inner sep=1pt, circle, fill=black] at (1.1,-1) (s2) {};
        \node[draw, inner sep=1pt, circle, fill=black] at (1.9,-1.6) (s3) {};
        \node[draw, inner sep=1pt, circle, fill=black] at (3.1,-1.1) (s4) {};
        \node[draw, inner sep=1pt, circle, fill=black] at (4.1,-0) (s5) {};
        \node[draw, inner sep=1pt, circle, fill=black] at (3.4,0.5) (s6) {};
        \node[draw, inner sep=1pt, circle, fill=black] at (2.9,0.) (s7) {};
        \node[draw, inner sep=1pt, circle, fill=red!80!black] at (3.25,-0.25) (goal) {};

        \draw[-latex, color=blue, thick] (s1) -- (s2);
        \draw[-latex, color=blue, thick] (s2) -- (s3);
        \draw[-latex, color=blue, thick] (s3) -- (s4);
        \draw[-latex, color=blue, thick] (s4) -- (s5);
        \draw[-latex, color=blue, thick] (s5) -- (s6);
        \draw[-latex, color=blue, thick] (s6) -- (s7);
        \draw[-latex, color=blue, thick] (s7) -- (goal);

        \end{tikzpicture}
    };

    \node[above=-0.1cm of initial-collection] {Trajectories of $\textcolor{blue}{\pi_\beta}$};
    
    \node[] (augmented-shortcuts) at (5,0) {
        \begin{tikzpicture}
        \node[draw, inner sep=1pt, circle, fill=black] at (1,0) (s1) {};
        \node[draw, inner sep=1pt, circle, fill=black] at (1.1,-1) (s2) {};
        \node[draw, inner sep=1pt, circle, fill=black] at (1.9,-1.6) (s3) {};
        \node[draw, inner sep=1pt, circle, fill=black] at (3.1,-1.1) (s4) {};
        \node[draw, inner sep=1pt, circle, fill=black] at (4.1,-0) (s5) {};
        \node[draw, inner sep=1pt, circle, fill=black] at (3.4,0.5) (s6) {};
        \node[draw, inner sep=1pt, circle, fill=black] at (2.9,0.) (s7) {};
        \node[draw, inner sep=1pt, circle, fill=red!80!black] at (3.25,-0.25) (goal) {};

        \draw[-latex, color=blue, thick] (s1) -- (s2);
        \draw[-latex, color=blue, thick] (s2) -- (s3);
        \draw[-latex, color=blue, thick] (s3) -- (s4);
        \draw[-latex, color=blue, thick] (s4) -- (s5);
        \draw[-latex, color=blue, thick] (s5) -- (s6);
        \draw[-latex, color=blue, thick] (s6) -- (s7);
        \draw[-latex, color=blue, thick] (s7) -- (goal);

        \draw[-latex, color=red!80!black, dotted,thick] (s2) to[bend left] (s4);
        \draw[-latex, color=red!80!black, dotted,thick] (s5) to[bend right] (goal);
        \draw[-latex, color=red!80!black, dotted,thick] (s4) to[bend right] (goal);
        \draw[-latex, color=red!80!black, dotted,thick] (s4) to[bend left] (s7);

        \end{tikzpicture}
    };

    \node[above=-0.1cm of augmented-shortcuts] {Shortcut augmented trajectories};

    \node[] (augmented-collected) at (10,0) {
        \begin{tikzpicture}
        \node[draw, inner sep=1pt, circle, fill=black] at (1,0) (s1) {};
        \node[draw, inner sep=1pt, circle, fill=black] at (1.1,-1) (s2) {};
        \node[draw, inner sep=1pt, circle, fill=black] at (1.9,-1.6) (s3) {};
        \node[draw, inner sep=1pt, circle, fill=black] at (3.1,-1.1) (s4) {};
        \node[draw, inner sep=1pt, circle, fill=black] at (4.1,-0) (s5) {};
        \node[draw, inner sep=1pt, circle, fill=black] at (3.4,0.5) (s6) {};
        \node[draw, inner sep=1pt, circle, fill=black] at (2.9,0.) (s7) {};
        \node[draw, inner sep=1pt, circle, fill=red!80!black] at (3.25,-0.25) (goal) {};

        \node[draw, inner sep=1pt, circle, fill=black] at (1.9,-1) (s2p) {};
        \node[draw, inner sep=1pt, circle, fill=black] at (2.9,0) (s3p) {};

        \draw[-latex, color=blue, thick] (s1) -- (s2);
        \draw[-latex, color=red!80!black, thick] (s2) -- (s2p);
        \draw[-latex, color=blue, thick] (s2p) -- (s3p);
        \draw[color=blue, dotted, thick] (s2) -- (s3);
        \draw[color=blue, dotted,thick] (s3) -- (s4);
        \draw[color=blue, dotted,thick] (s4) -- (s5);
        \draw[color=blue, dotted,thick] (s5) -- (s6);
        \draw[color=blue, dotted,thick] (s6) -- (s7);
        \draw[-latex,color=blue, thick] (s7) -- (goal);

        \end{tikzpicture}
    };

    \node[above=-0.1cm of augmented-collected] {Trajectories of $\textcolor{blue}{\pi_\beta}$ with
    $\textcolor{red}{a_\theta}$};

    \draw[-latex, thick] (0, -1) to[bend right]
        node[midway, fill=white,align=center] {Compute\\shortcuts} (4, -1);

    \draw[-latex, thick] (5, -1) to[bend right]
node[midway, fill=white,align=center] {Train\\augmentor $\textcolor{red}{a_\theta}$} (9, -1);

\end{tikzpicture}
}
    \caption{
Overview of \ourmethod{}.}\label{fig:overview}
\end{figure}

In this paper, we study logging policy augmentation in the context of
\emph{active positioning problems} that capture both partial observability and 
fine tolerance demands that make online RL particularly costly, while also reflecting the prevalence
of deterministic procedures in practice, making them an ideal testbed for offline RL in
general and logging augmentations in particular.  
Additionally, their contextual and geometric structure enables a theoretically grounded analysis of
when and why augmentations are beneficial.
They require placing an object precisely at a
desired position by an end-effector, spanning a wide range of challenging RL problems, from high-precision
positioning tasks as alignments of lens
systems~\cite{burkhardt2025relign} and camera and telescope
assembly~\cite{camera_alignment,telescope_alignment}, the alignments of laser
optics~\cite{laser_alignment,interferobot}, to robot manipulation
tasks~\cite{fetch}.

\subsection{Contributions}

We introduce \emph{LIFT}, short for logging improvement via fine-tuned
trajectories, a framework that enhances punctual data collection for offline RL. 
Specifically, we propose a novel augmentation scheme (Section~\ref{sec:lift}) that keeps the logging
policy in control while enabling optimistic probing by an \emph{augmentor} trained while data is collected.
The augmentor’s goal is to skip redundant and unnecessary sub-trajectories during collection and to
smooth hand-offs between itself and the logging policy.
A key challenge here is that the augmentor has to suggest beneficial actions while being
trained with very limited data. A central innovation of our work is to leverage the 
geometric structure of the logged trajectories to identify \emph{shortcuts}, that is, actions that
point towards states with higher value. Identifying shortcuts is non-trivial in general due to distortions in
the dynamics and the partial observability. We prove in Section~\ref{sec:theory} under which conditions such shortcuts can be reliably identified 
in logged data, and we devise an algorithm to extract them from this data (Algorithm~\ref{alg:shortcut_computation}).
Finally, Section~\ref{sec:experiments} presents a systematic study that underlines the strength and
generality of our approach by analyzing the effect of the logging policy, transition behavior, dimensionality, and
informativeness of observations on policy performance across a diverse class of active
positioning tasks. We implemented the shortcut augmentation in d3rlpy~\cite{d3rlpy}, following its transition picker protocol, which allows our
static augmentation method to be integrated into any RL algorithm implemented in d3rlpy by adding a single line of code. The source code and 
integration examples are available on GitHub.\footnote{\url{https://github.com/HS-Kempten/lift}}

\subsection{Related Work}\label{s:related_work}

A central challenge in offline RL is overestimating values for out-of-distribution actions.
Methods address this either by constraining the learned policy toward the logging distribution or b
learning pessimistic value functions.  Representative approaches include behavior regularization via
BC losses or divergence penalties~\citep{bcq,td3_bc,rebrac}, pessimistic critics~\citep{cql}, or expectile-based policy
extraction~\citep{iql}. Methods depending on
regularizations are sensitive to hyperparameters and they often limit the policy to stay close to
the behavior, for instance due to safety constraints, which can be detrimental if the behavior is
highly suboptimal. 
Moreover, several studies note that algorithm
performance is highly sensitive to dataset composition~\citep{d4rl, hong2023harnessing}, that is, mixing suboptimal trajectories with
expert data. Prior work has
studied intensively the importance of high-coverage~\cite{exorl,pmlr-v267-wagenmaker25a} and
expertness of datasets~\cite{offline_rl_vs_bc,corrado2024guda} for offline RL. 
This has been underpinned by the investigations in~\cite{dataset-perspective}, where scores are
designed that measure exploitation and exploration capabilities of datasets and how these affect
algorithmic performance of offline RL methods. While \citet{ghugare2024closing} discuss limitations in combinatorial generalization (‘stitching’) 
from an algorithmic perspective, our work addresses a complementary problem at the data level through trajectory-level augmentation.
Increasing the dataset diversity via data augmentations is another line of work to mitigate 
narrow data distributions. 
In~\cite{HER}, an augmentation scheme for sparse reward in robotic manipulation tasks is proposed
that re-labels goals and states in logged trajectories to create additional successful transitions.
Augmentations for problems with image observations have been studied
extensively in the literature, where it was shown that rather simple image
augmentations~\cite{rl_with_augmentations, s4rl}, such as random cropping,
or utilizing causal
techniques~\cite{counterfactual_augmentation} can significantly improve sample efficiency.
Recently, diffusion-based techniques have been proposed 
that generate synthetic trajectories in order to make offline RL more robust~\cite{diffstitch,gta,
synthetic_experience_replay}.
In contrast to purely \emph{offline} augmentations on static datasets, hybrid schemes that actively
enhance data collection are more relevant to our work. A common hybrid approach warm-starts online
reinforcement learning from an offline-trained policy and continues training with newly collected
online data. Prior work shows that, combined with careful sampling schemes and network architectures
\citep{rlpd} or policy regularization \citep{overcoming_exploration_in_rl}, this can yield strong
initializers for online learning. Nevertheless, these methods still require rather long online
fine-tuning or high-quality offline datasets, neither of which is typically available in active positioning
tasks. A more subtle scheme is to let an expert guide the data collection process, like in
GuDA~\cite{corrado2024guda}, where human-guidance is interleaved to direct trajectories toward
success. Another relevant line of work is to weave online transitions into logging policies as in
iterative offline RL (IORL)~\citep{zhang2023active}. Here, exploratory actions are injected to
discover unexplored regions in state-action space while training an offline RL agent on the
generated trajectories. This approach is discussed in Section~\ref{sec:lift}. Our approach is
similar in spirit, but instead of exploring we want to exploit shortcuts in the trajectories
to make hand-offs seamless and effective.

\section{Active Positioning}\label{sec:active_positioning}

In this section, we introduce the specific framework for active positioning problems building
upon the framework for active alignments introduced in~\cite{burkhardt2025relign}. There,
active positioning problems are modelled as an \emph{episodic} and \emph{contextual} POMDP~\cite{context-mdp}. 
Specifically, the state is decomposed in the current position $s\in\cP$ with $\cP$ a bounded subset
of $\RR^m$ and a static context parameter $W\in\cW$, that is $\cS=\cP\times\cW$. The actions can be
selected from a subset $\cA$ of $\RR^d$. 
Applying an action
$a\in\cA$  at state $(s, W)$ gives the new state $(s', W)$
with $s'=f(s, a, W)$, where $f:\cP\times\cA\times\cW\to\RR^d$ is a parametrized \emph{distortion function}.
Throughout we assume that $f(s,0,W)=s$. 
Our running example is $f(s, a, W)=s+W\cdot a$ with $W\in\RR^{d\times d}$, but we also consider
non-linear and non-continuous distortions. Importantly, as $W$ stays constant
throughout each episode, so is the extent of the distortion. One can think of $W$ as variances
introduced by the gripping of an object, variances within an object, or conditions of the goal to be
reached. In robotic arm positioning, for instance, $W$ can model the imprecision of the
end-effector due to load or joint friction as well as where the target $s_W$ is located.
\begin{figure}[ht]
    \centering
\scalebox{0.67}{
\input{./images/relign.tex}
}
\caption{Active positioning of a lens systems~\cite{burkhardt2025relign} (left) and an
end-effector~\cite{fetch} (right).}\label{f:active_positioning}
\end{figure}

In each episode, the goal is to navigate from a random initial position $s_0$ and randomized context
$W$ to a terminal state $s_W\in\RR^d$. The reward observed when applying $a$ at
$(s, W)$ is $R(s, a, W) = -\|f(s, a, W)-s_W\|$, i.e.\
the negative remaining distance to the terminal state.
An episode ends once the state is sufficiently close to $s_W$ or an upper limit of 
steps is reached. Formally, the terminal states are all within the set
 $\{(s, W)\in S:
\|s-s_W\|\le\theta\}$.  Typically, $W$ cannot be observed directly, often even $s$ cannot.
Instead, an often high-dimensional and noised output
$O(s, W)\in\mathcal{O}$ is observed, which is controlled by a conditional probability
density function depending on $s$ and $W$. In robotic arm positioning, the observation can come from
a camera mounted on the end-effector or from sensors measuring forces and torques.
We call $(\cP, \cW, \cO, f, \gamma)$ an \emph{active positioning
problem}. This framework covers various industrial use cases, from robot arm positioning, to active
alignments of optical devices (Figure~\ref{f:active_positioning}).

Although active positioning problems can also be considered as black-box optimization
problems~\cite{burkhardt2025relign}, they are inherently RL problems where symmetries and
ambiguities in the need to be actively explored. 
For instance, the observation space is typically highly symmetric and context-dependent: states $s$
and $s'$ that are far apart can yield very similar observations $O(s, W) \approx O(s', W)$, while the
same state can produce very different observations $O(s, W)$ and $O(s, W')$ under different contexts.
Additionally, safety constraints and physical limitations often restrict the action space $\cA$ so
that the optimal state cannot be reached in one step and a sequence of informed actions is required.
In the RL formulation, a \emph{policy} ~$\pi:\cA\times\cO\to\RR$ is a mapping of
observations and actions to likelihood and the dynamics of the combined system works as follows:
At a given state $(s, W)$, $O(s, W)$
is observed, an action $a$ is sampled from 
$\pi(\cdot, O(s, W))$, and the system moves to the new state $s'=f(s, a, W)$.
Note that $a$ and $s$ do not need to have same dimensionality.
Starting from $(s_0, W)\in\cS$, the dynamics 
yields a trajectory $(s_0, W),\ldots,(s_k, W)$.
The goal is to find~$\pi$ maximizing
$J(\pi):=\mathbb{E}_{s_0, W}\left[\sum_{i=0}^{k} -\gamma^i \|s_i-s_W\|\right]$, where
$\gamma\in(0,1)$ is a \emph{discount factor}. 
Clearly,
$J(\pi)=\mathbb{E}_{s_0, W} [V^\pi(s_0, W)]=\mathbb{E}_{s_0} [V^\pi(s_0)]$
with $V^\pi$ the state-value function and
$V^\pi(s) := \mathbb{E}_{W\sim\cW} [V^\pi(s, W)]$. 

\section{Theory of Shortcut Augmentations}\label{sec:theory}
In active positioning, good trajectories reach the optimal position in as few steps as possible.
Although most logging policies used in applications visit states that are close to the optimal
state, they often produce long and redundant trajectories. 
Our core idea is to train agents on synthetic trajectories distilled from these imperfect data,
which are more direct and goal-reaching.
Intuitively, we want the agent to \emph{skip} parts of the trajectory that do not add much value — for
example, going straight instead of replicating zig-zag movements or detours present in the logged
data (Figure~\ref{fig:overview}). However, improving logged trajectories is not straightforward.
For instance, assume a collected trajectory of $\pi_\beta$ contains a sub-trajectory $(s_i, W),
(s_{i+1}, W), \ldots, (s_j, W)$ with actions $a_i, \ldots, a_{j-1}$, representing a long detour,
like a zig-zag movement, from $s_i$ to $s_j$. Clearly, going directly from $s_i$ to $s_j$ would yield a
trajectory with higher return. However, naively applying the accumulated action $a = a_i + a_{i+1} +
\ldots + a_{j-1}$ at $s_i$ will not necessarily land exactly at $s_j$ due to distortions in the dynamics
induced by~$f$. Even small misplacements, that is ending up close to $s_j$ but not exactly at $s_j$, can cause
significant value degradation if the value function $V^{\pi_\beta}$ is not stable in the vicinity of $s_j$. Worse,
applying $a$ at $s_i$ may even move us in the opposite direction, away from $s_j$, with no guarantee that
the new state has a higher value than $s_i$. Here, the length of the action $a$, the value gap between
$s_i$ and $s_j$, the stability of $V^{\pi_\beta}$ around $s_j$, and the distortion in the dynamics at
$s_i$ all play a role. 
In this section, we identify conditions under which the accumulated action $a$ is guaranteed to be beneficial.
All proofs are in Section~\ref{s:proofs}. 
We call a policy $\pi$ \emph{distance-improving} if for all $W\in\cW$ we have for two subsequent
states $(s_i, W)$ and $(s_j, W)$ with $i<j$ visited by the policy that $\|s_j-s_W\|<\|s_i-s_W\|$.
In other words, the reward along a trajectory of $\pi$ is strictly increasing. We restrict to deterministic logging policies $\pi$, so that the contextual but deterministic dynamics given
by~$f$ implies that $V^\pi(s, W)$ is exactly the return of $\pi$ starting from $(s, W)$.

\begin{prop}\label{prop:augmentations}
    Let $\pi$ be distance-improving and $(s, W), (s', W)\in\cS$ on a
    trajectory where $(s, W)$ is prior to
    $(s',W)$, then
    $\gamma V^\pi(s', W)- V^\pi(s, W)\ge \|s'-s_W\|$.
\end{prop}
Focusing on distance-improving logging policies allows us
to formalize what it means for an action to be beneficial.
\begin{defn}\label{defn:shortcut}
    Let $\pi$ be a policy, $(s, W)\in\cS$ a state, and
    $a\in\cA$ an action with $s'=f(s, a, W)$. If 
    $\gamma V^{\pi}(s', W) -V^\pi(s, W)\ge \|s'-s_W\|$, then
    $a$ is a $\pi$-\emph{shortcut} at $(s, W)$.
\end{defn}
Note that shortcuts depend on the latent information $W$, not~$s$ alone. The
remainder of this section studies how to find shortcuts in offline trajectories. To do so, consider
a short trajectory $(s_0, W), (s_1, W), (s_2, W)$ from a distance-improving policy $\pi$ with
actions $a_0$ and $a_1$ (Figure~\ref{f:movement_uncertainty}). Clearly, any action $a$ with
$s_2=f(s_0, a, W)$ is a $\pi$-shortcut and thus beneficial. However, because of non-linearities in~$f$, applying $a_0+a_1$ at $s_0$ is not guaranteed to reach $s_2$. 
Hence, we must ensure that $a_0+a_1$ leads near $s_2$ requiring to control the placement errors induced by $f$.
For linear dynamics $f(s,a,W)=s+W\cdot a$ with $W\in\RR^{m\times d}$, any accumulated action is a
shortcut, irrespective of $V^\pi$:
\begin{prop}\label{prop:shortcuts_for_linear_dynamics}
Let $f(s, a, W)=s+W\cdot a$, $(s_i, W)$, $(s_j,
W)$ with $i<j$ on a trajectory of a 
distance improving policy $\pi$ and $a_i,\ldots,a_{j-1}$ the actions $\pi$ applied to
get from $s_i$ to $s_j$. 
Then $\sum_{k=i}^{j-1}a_k$ is a $\pi$-shortcut for $s_i$.
\end{prop}

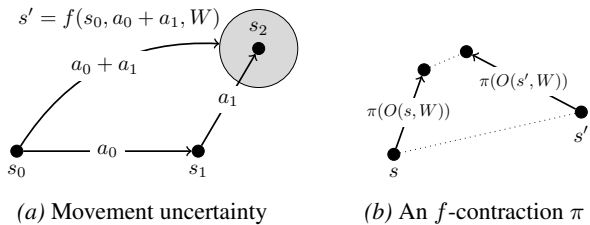
\begin{figure}[ht]
    \centering
\subfloat[Movement uncertainty]{\label{f:movement_uncertainty}
    \scalebox{0.8}{
    \begin{tikzpicture}

        \node[draw, inner sep=13pt, circle, fill=gray!30!white] at (4,1.7) (s3) {};
        \node[anchor=west] at (-0.1,2.2)  {$s'=f(s_0, a_0+a_1, W)$};

        \node[draw, inner sep=2pt, circle, fill=black, label=below:$s_0$] at (0,0) (s0) {};
        \node[draw, inner sep=2pt, circle, fill=black, label=below:$s_1$] at (3,0) (s1) {};
        \node[draw, inner sep=2pt, circle, fill=black, label=above:$s_2$] at (4,1.7) (s2) {};

        \draw[->,thick] (s0) to node[midway, fill=white] {$a_0$} (s1);
        \draw[->,thick] (s1) to node[midway, fill=white] {$a_1$} (s2);
        \draw[->,thick, bend left] (s0) to node[midway, fill=white] {$a_0+a_1$} (s3);

\end{tikzpicture}}
}
\hspace{0.5cm}
\subfloat[An $f$-contraction $\pi$]{
    \scalebox{0.8}{
    \begin{tikzpicture}

        \node[draw, inner sep=2pt, circle, fill=black, label=below:$s$] at (0,0) (s0) {};
        \node[draw, inner sep=2pt, circle, fill=black] at (0.5,1.4) (s2) {};
        \node[draw, inner sep=2pt, circle, fill=black, label=below:$s'$] at (3.1,0.7) (s1) {};
        \node[draw, inner sep=2pt, circle, fill=black] at (1.2,1.7) (s3) {};

        \draw[->,thick] (s0) to node[midway, fill=white, scale=0.8] {$\pi(O(s, W))$} (s2);
        \draw[->,thick] (s1) to node[midway, fill=white, scale=0.8] {$\pi(O(s', W))$} (s3);

        \draw[dotted] (s0) to (s1);
        \draw[dotted] (s2) to (s3);

    \end{tikzpicture}
}
}
\caption{Interactions of policy with movement dynamics.}
\end{figure}

Extending Proposition~\ref{prop:shortcuts_for_linear_dynamics} to non-linear dynamics~$f$ is not
trivial. Generally, we want to have that accumulating actions along a trajectory does not lead
to too much placement uncertainty, which is typically the case in real-world positioning problems.
We formalize this as follows:

\begin{defn}[Linear placement-errors]\label{def:linear_placement_errors}
    A distortion function $f$ has \emph{linear placement-errors} (LPE) if there is a constant $L_f$
    so that for
    any action-chain $a_0,\ldots,a_{k-1}$ executed from $(s_0, W)$ with
    $s_i=f(s_{i-1}, a_{i-1}, W)$,
    we have:
$\|f(s_0, \sum_{i=0}^{k-1}a_i, W) - s_k\|\le L_f\cdot \sum_{i=0}^{k-1}\|a_i\|$.
\end{defn}

Intuitively, the LPE property means that although a system distorts movements, the mismatch
introduced when regrouping actions cannot grow faster than linearly with the size of
the path taken. This actually includes a wide range of functions where the distortion depends
on the state
only:

\begin{prop}\label{prop:lpe_for_bounded_transform}
    Let $f(s, a, W)=s+g(s, W)\cdot a$ with $g:\cS\to\RR^{m\times d}$ a bounded
    matrix-function. Then $f$ has
    LPE with $L_f=2\cdot\sup_{\cS}\|g\|$.
\end{prop}

As we will see, when the distortion term also depends on the action, i.e. $g(s, a, W)$, things
become more involved for small actions $a$ even if $g$ is bounded and LPE does not follow without
additional assumptions (see Section~\ref{s:movement_distortion}).
In Proposition~\ref{prop:linear_placement_erros}, we introduce an even stronger property which
suffices to imply LPE for distortion functions of common active positioning problems, like
linear movement dynamics. More specifically, it follows directly that a linear movement-dynamics of
the form $f(s, a, W)=s+Wa$ has LPE with $L_f=0$. 

Having gathered a notion of placement errors, we now need to control the stability of the value
function. Specifically, even when we can precisely reach $s_j$ from $s_i$, the value function
$V^\pi$ can change drastically in the vicinity of $s_j$, making it hard to guarantee that applying
the accumulated action $a$ at $s_i$ is indeed beneficial. To control this, we have to impose good
properties on $V^\pi$. We
call a value function $V:\cS\to\RR$ \emph{$L_V$-Lipschitz continuous} if for all $(s, W), (s', W)\in\cS$
we have $|V(s, W)-V(s', W)|\le L_V\cdot\|s-s'\|$. This is the final ingredient to prove our main statement:

\begin{thm}\label{thm:shortcuts}
Let $\pi$ be distance improving, $V^\pi$ is $L_V$-Lipschitz continuous and, let $f$ has
$L_f$-placement errors. Let $(s_i, W)$ and $(s_j,
W)$ on a trajectory of $\pi$ and let $a=\sum_{k=i}^{j-1}a_k$ be the sum of the chain of actions $\pi$ undertook to
get from $s_i$ to $s_j$. Then
$a$ is a $\pi$-shortcut for $s_i$ if
\begin{equation*}
\begin{split}
\gamma\cdot V^\pi(s_j, W) - V^\pi(s_i, W) - \|s_j-s_W\| \\
\ge (\gamma\cdot L_V+1)\cdot L_f\cdot \sum_{k=i}^{j-1}\|a_k\|
\end{split}
\end{equation*}
\end{thm}

Note that if $j=i+1$, the left hand side in Theorem~\ref{thm:shortcuts} is zero. However, in that
case, $a_i$ is, by definition, the only shortcut from $(s_i, W)$ to $(s_j, W)$ as its the direct
connection from $s_i$ to $s_j$.
Proposition~\ref{prop:shortcuts_for_linear_dynamics} for $f(s, a,
W)=s+W\cdot a$ arises as a special case of
Theorem~\ref{thm:shortcuts} because $L_f=0$ implies that the right-hand side is $0$ and the
left-hand side is always non-negative due to Proposition~\ref{prop:augmentations}. However, 
Theorem~\ref{thm:shortcuts} requires $V^\pi$ to be Lipschitz continuous, where no assumptions on
$\pi$ are necessary in Proposition~\ref{prop:shortcuts_for_linear_dynamics}.
The next condition helps to ensure that $V^\pi$ is indeed
Lipschitz continuous (see Proposition~\ref{prop:lipschitz_value_function}), which requires a
beneficial interplay with~$f$:

\begin{defn}[$f$-contraction]
    A policy $\pi$ is an $f$-\emph{contraction} if 
    for all $(s, W),(s', W)$
    with respective observations with $o=O(s, W)$ and $o'=O(s', W)$, we have
$$\|f(s, \pi(o), W)-f(s',\pi(o'), W)\| \le \|s-s'\|.$$
\end{defn}

\begin{cor}\label{cor:shortcuts}
Let $\pi$ be distance improving $f$-contraction and let $f$ have LPE with
constant $L_f$. Let $(s_i, W)$ and $(s_j,
W)$ on a trajectory of $\pi$ and let $a=\sum_{k=i}^{j-1}a_k$ be the sum of the chain of actions $\pi$ undertook to
get from $s_i$ to $s_j$. Then
$a$ is a shortcut for $s_i$ if
\begin{equation*}
    \gamma\cdot V^\pi(s_j, W) - V^\pi(s_i, W) -\|s_j-s_W\|\ge \frac{L_f}{1-\gamma}\cdot
\sum_{k=i}^{j-1}\|a_k\|.
\end{equation*}
\end{cor}

Being an $f$-contraction is a stronger requirement than mere distance improvement. We refer to
Section~\ref{app:contractions} for a discussion and examples of $f$-contractions and Lipschitz value
functions in real-world policies. In practice, many active positioning policies do not satisfy the
contraction property globally, yet this is not required for identifying useful shortcuts as shown in
our experiments.

\section{Trajectory Augmentation via \ourmethod{}}\label{sec:lift}

The idea of iterative reinforcement learning is to enrich logging policies with exploratory steps
while collecting data~\citep{zhang2023active}, mostly in order to improve coverage of the
state-action space. Specifically, an \emph{uncertainty
model} $E_\theta(s, a)$ is trained with $E_\theta(s, \cdot)$ a probability distribution on $\cA$ for each $s\in S$. Given a
dataset $D$, $E_\theta$ is trained by minimizing
$\mathbb{E}_{(o,a)\sim D}\big[ -\log(E_\theta(s,a)) + \mathcal{R}(\theta)\big]$
with $\mathcal{R}(\theta)$ a regularization term. Intuitively, $E_\theta(s, a)$ can be seen as the
probability that action $a$ has been seen for state $s$ in $D$. Actions with small probability
$E_\theta(s, a)$ at state $s$ are considered as exploratory actions and should be selected according
to some fixed probability $p$ enriching a given logging policy $\pi_\beta$ during rollout.
These \emph{exploratory actions} are rather rare and thus help keeping the system safe and naturally
close to the logging policy $\pi_\beta$ that generated the data. 
Although this approach seems appealing, a central part has been
underexplored in current literature, namely that static logging policies may not deal well with
intermediate exploratory steps. 
In practice, arbitrary exploratory steps may lead to states from which the logging policy cannot
recover well, resulting in lower overall returns.
We build upon this idea, but instead of selecting actions that have not been seen in the data, we
advocate to train a $Q$-function $Q_\theta$ on some initial dataset $D$ and select actions having
high $Q$-values. Formally, we set $a_\theta(s, a)=\max_{a'\in\cA}Q_\theta(s, a')$ where $Q_\theta$
can be trained with any offline RL method, like CQL or IQL. We call $a_\theta$ an \emph{augmentor}.
By that, we aim to enrich the dataset with actions that are likely to be beneficial for $\pi_\beta$
in the sense of higher returns. 
While this idea is quite universal and it remains unclear how actions
that ease hand-offs look like in general. 
Moreover, in order that the augmentor provides useful steps, it has to be trained well already with limited
data. The idea of LIFT is to show the augmentor data of \emph{good behavior} by applying augmentation to the
logged data that emphasizes such behavior. Clearly, when $a_\theta$ to suggest at $o=O(s, W)$
$\pi_\beta$-shortcuts (Definition~\ref{defn:shortcut}), a logging policy with higher return can be obtained
by combining them (see Proposition~\ref{prop:augmented_is_better} for details):
\begin{equation*}
    \pi_{\textnormal{aug}}(o) :=
\begin{cases}
    a_\theta(o) \quad\textnormal{if } a_\theta(o)\textnormal{ is a $\pi_\beta$-shortcut at } (s,
    W)\\
\pi_\beta(o) \quad\textnormal{otherwise}
\end{cases}\hspace{-0.2cm}.
\end{equation*}
This can be seen as a specialization of the policy improvement
theorem~\citep[Section~4.2]{sutton2018reinforcement} to active positioning.
For the remainder, we discuss how to train $a_\theta$ in
order that it suggests $\pi_\beta$-shortcuts for active positioning problems. However, we want to
emphasize that LIFT in general is not tied to this form of backbone-augmentations.

Theorem~\ref{thm:shortcuts} gives a condition when and how to augment a
trajectory 
$(o_0, a_0, r_0),\ldots,(o_n, a_n, r_n)$ with latent states $s_i=f(s_{i-1}, a_{i-1}, W)$,
observations $o_i=\mathcal{O}(s_i, W)$, rewards
$r_i=-\|s_{i+1}-s_W\|$, and actions $a_i=\pi_\beta(o_i)$
from a logging policy $\pi_\beta$. To convey them into a practical algorithm, let $C\in\RR_{\ge 0}$ be a constant and let
$G_i=V^{\pi_\beta}(s_i, W)=\sum_{k=i}^n\gamma^{k-i}r_{k}$ be the returns of $\pi_\beta$.
Now, take any pair $(i,j)$ with $i<j$, let $\hat a=\sum_{k=i}^{j-1}a_i$ be a shortcut candidate and check if
$\gamma G_j - G_i +r_{j-1}\ge C\cdot\sum_{k=i}^{j-1}\|a_k\|$
with some constant $C$ holds true.
Clearly, without prior information on $f$ and $\pi_\beta$, the exact value of $C$ remains unclear, and thus
it has to be considered a regularization hyperparameter of our method. If $C=0$, all pairs are considered shortcuts, if $C$ is large, only very few pairs where high reward
is gained in a few short steps are considered shortcuts. 
If the inequality is valid for $(i,j)$, we can assume that $\hat a$ is a shortcut and ideally, we
would add the tuple
$(o_i, \hat a, -\|s_j'-s_W\|, o_j')$ with $s_j'=f(s_i, \hat a, W)$ and $o_j'=O(s_j', W)$ to
the dataset.
However, due to the movement uncertainty, there is a gap between the position $s_j'$ the shortcut
leads to and the observed state $s_j$. Particularly, the image observation
$O(s_j', W)$ and the reward $-\|s_j'-s_W\|$ differ from the actually observed ones, namely $o_j$
and $r_{j-1}$. We argue, however, that in many practical applications, this gap is small, for instance if $L_f=0$ as in linear movement dynamics $f(s, a, W)=s+W\cdot
a$ (see Proposition~\ref{prop:shortcuts_for_linear_dynamics}). Thus, we add 
$(o_i, a, r_{j-1}, o_j)$ to the training dataset. Algorithm~\ref{alg:shortcut_computation}
summarizes our shortcut sampling procedure, and we want to
emphasize that it can be added to any offline RL method that samples from an offline dataset, like
to minimize the Bellman error or related temporal difference errors as in CQL. Note that for a given
input tuple, the runtime of Algorithm~\ref{alg:shortcut_computation} is linear in the trajectory
length.
Observe that the synthetic shortcuts are
only used to obtain the augmentor $a_\theta$, which in turn is only used to fine-tune the logging policy, and the
collected dataset consists of real data only. The precise procedure is described in
Algorithm~\ref{alg:data_augmentation}.
For that, they must have good hand-over properties and thus we
augment the dataset $D$ with shortcuts computed via Algorithm~\ref{alg:shortcut_computation} when
training $Q_\theta$. 
\begin{algorithm}[H]
  \caption{Shortcut sampling}
  \begin{algorithmic}[1]
    \REQUIRE $C\ge 0$, $i\in[n]$, $\{(o_0, a_0, r_0),\ldots,(o_n, a_n, r_n)\}$
    \ENSURE Tuple $(o_i, \hat a, r_{j-1}, o_j)$
    \STATE Compute returns $G_0\ldots, G_n$ for trajectory
    \STATE $S= ( )$
    \FOR{$j=i+1\cdots n$}
      \STATE $\hat a_i \gets \sum_{k=i}^{j-1}a_k$
      \IF{$\gamma G_j- G_i+r_{j-1}\ge C\cdot \sum_{k=i}^{j-1}\|a_k\|$ and $\hat a_i\in\cA$}
        \STATE Add $(o_i, \hat a_i, r_{j-1}, o_j)$ to $S$
      \ENDIF
    \ENDFOR
    \STATE Let $m=|S|$ and let $\hat r=(\hat r_1,\ldots, \hat r_m)$ be the rewards of the tuples in $S$
    \STATE Let $\rho\sim \hat r - \min_{i}\hat r_i$ a mass function
    \STATE Sample $(o_i, \hat a_i, r_{j-1}, o_j)$ from $S$ w.r.t.~$\rho$
    \STATE \textbf{return} $(o_i, \hat a_i, r_{j-1}, o_j)$
  \end{algorithmic}
  \label{alg:shortcut_computation}
\end{algorithm}

    \begin{algorithm}[H]
      \caption{\ourmethod{}}
      \begin{algorithmic}[1]
          \REQUIRE $\pi_\beta, n\in\NN, a_\theta, p\in[0,1]$
        \ENSURE Dataset $D$ with $n$ trajectories
        \STATE Initialize $D\leftarrow\emptyset$
        \REPEAT
          \STATE Sample $o_0$ from environment
          \STATE Set $\mathrm{done}=\mathrm{false}, \tau = ( ), i=0$
          \WHILE{$\mathrm{done}$ is false}
            \STATE $a_i=\pi_\beta(o_i)$
            \IF{$\mathrm{rand}() \le p$}
              \STATE $a_i= a_\theta(o_i, a_i)$
            \ENDIF
            \STATE $o_{i+1}, r_i, \mathrm{done} = \texttt{env.step}(a_i)$
            \STATE Reset $\pi_\beta$ at $o_{i+1}$ (if $a_i$ was augmented by $a_\theta$)
            \STATE Add $(o_i, a_i, r_i)$ to $\tau$, $i=i+1$
          \ENDWHILE
          \STATE Add trajectory $\tau$ to $D$
          \IF{train augmentor}
            \STATE Train $a_\theta$ on $D$ with help of Algorithm~\ref{alg:shortcut_computation}
          \ENDIF
        \UNTIL{$|D| = n$}
        \STATE \textbf{return} $D$
      \end{algorithmic}
      \label{alg:data_augmentation}
    \end{algorithm}

\vspace{-0.5cm}
\section{Experiments}\label{sec:experiments}

Our experiments address two main questions: Can shortcut augmentations improve pure offline
RL and can they be leveraged during data collection by training an augmentor in
comparison to warm-start RL? We test different distortions $f$, observation
types $\mathcal{O}$, and levels of logging expertness. 

\subsection{Environments}

In order to analyze different movement distortions and observation types in isolation, we conducted
our experiments in semi-realistic active positioning environments designed
to keep real world characteristics and entail small sim-to-real gaps. Throughout, we use
$-\|s-s_W\|$ as reward signal, which is easy to compute in simulations, as one typically has access
to latent information $(s, W)$. When data is coming from a real system,
In real systems, this signal can easily be added in hindsight to finished episode once $s_W$ is
uncovered by the logging policy.

\subsubsection{Movement distortions}\label{s:movement_distortion}

We consider different movement distortions, some of them have linear
forms, like $f_{\text{blend}}$ and $f_{\text{rot}}$ both with $L_f=0$. We also use
non-linear distortions, like $f_{\text{scale}}$ and $f_{\text{sin}}$ which have LPE with $L_f>0$ and
one non-continuous distortion $f_{\text{regrot}}$ also having
LPE which is not contracting. Moreover, we test a dynamics $f_{\text{sqrt}}$ that does not satisfy the LPE property. We refer to
Section~\ref{app:properties_of_distortions} for their precise mathematical definitions and
proofs of their properties. Figure~\ref{f:movement_distortions} illustrates an
overview of the different distortions in two dimensions. 

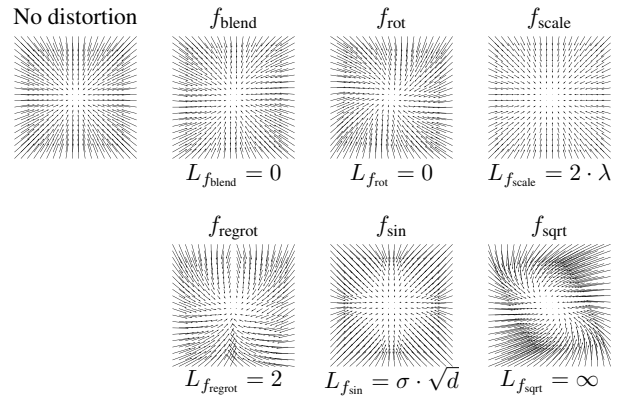
\begin{figure}[ht]
    \centering
    \scalebox{1.1}{
    \begin{tikzpicture}

        \node[] (void) at (0,0) {\includegraphics[width=1.7cm]{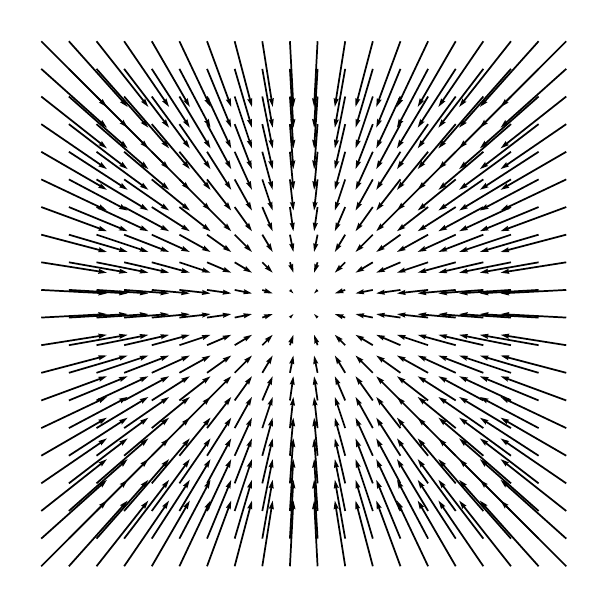}};
        \node[] (blend) at (1.9,0) {\includegraphics[width=1.7cm]{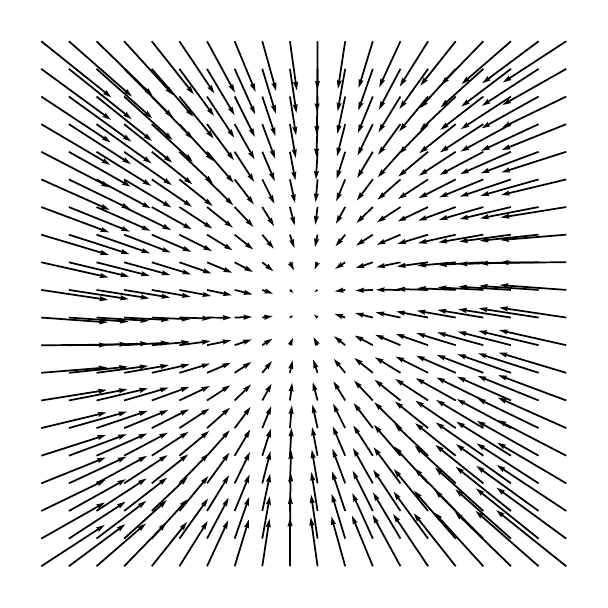}};
        \node[] (rot) at (3.8,0) {\includegraphics[width=1.7cm]{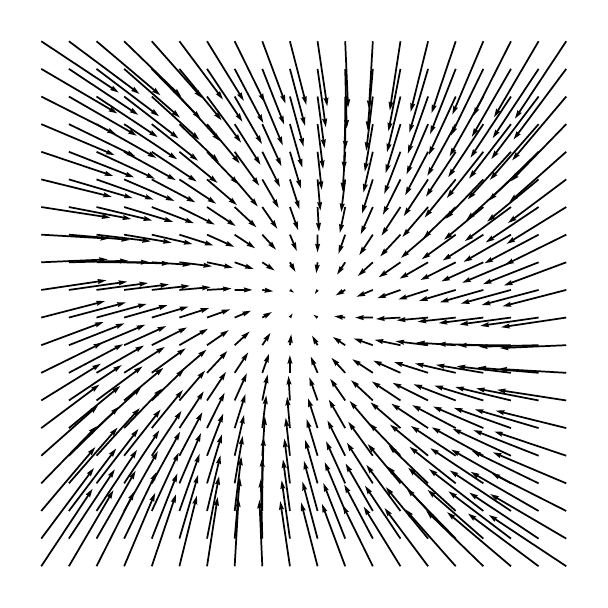}};
        \node[] (scale) at (5.7,0) {\includegraphics[width=1.7cm]{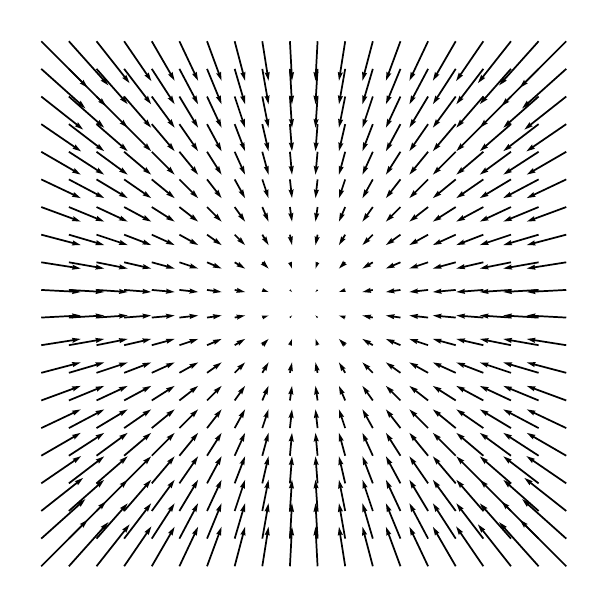}};

        \node[] (regrot) at (1.9,-2.5) {\includegraphics[width=1.7cm]{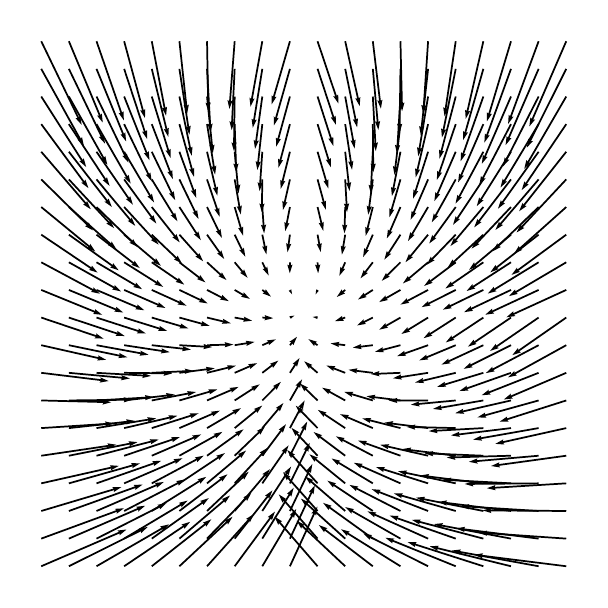}};
        \node[] (sin) at (3.8,-2.5) {\includegraphics[width=1.7cm]{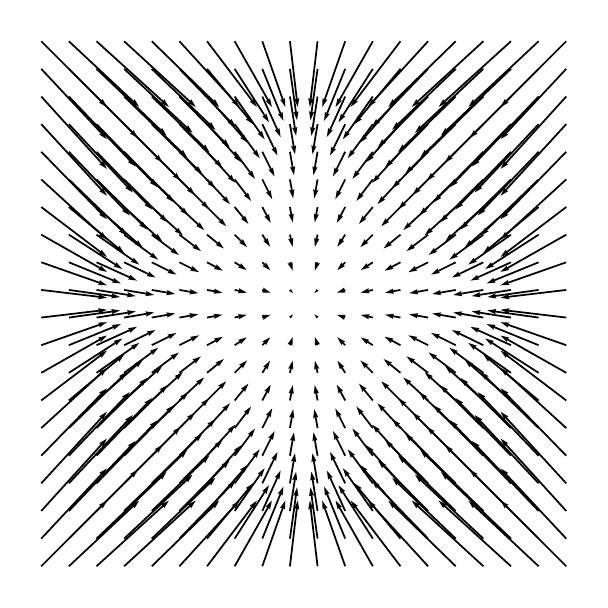}};
        \node[] (sqrt) at (5.7,-2.5) {\includegraphics[width=1.7cm]{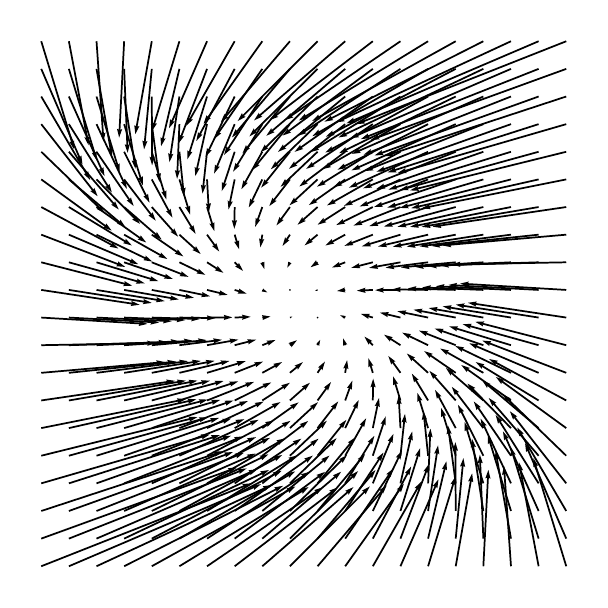}};

        \node[above=0.2cm of void, anchor=north, scale=0.8] {No distortion};
        \node[above=0.2cm of blend, anchor=north, scale=0.8] {$f_{\text{blend}}$};
        \node[above=0.2cm of rot, anchor=north, scale=0.8] {$f_{\text{rot}}$};
        \node[above=0.2cm of scale, anchor=north, scale=0.8] {$f_{\text{scale}}$};
        \node[above=0.2cm of regrot, anchor=north, scale=0.8] {$f_{\text{regrot}}$};
        \node[above=0.2cm of sin, anchor=north, scale=0.8] {$f_{\text{sin}}$};
        \node[above=0.2cm of sqrt, anchor=north, scale=0.8] {$f_{\text{sqrt}}$};

        \node[below=0.2cm of blend, anchor=south, scale=0.8] {$L_{f_{\text{blend}}}=0$};
        \node[below=0.2cm of rot, anchor=south, scale=0.8] {$L_{f_{\text{rot}}}=0$};
        \node[below=0.2cm of scale, anchor=south, scale=0.8] {$L_{f_{\text{scale}}}=2\cdot\lambda$};
        \node[below=0.2cm of regrot, anchor=south, scale=0.8] {$L_{f_{\text{regrot}}}=2$};
        \node[below=0.2cm of sin, anchor=south, scale=0.8] {$L_{f_{\text{sin}}}=\sigma\cdot\sqrt{d}$};
        \node[below=0.2cm of sqrt, anchor=south, scale=0.8] {$L_{f_{\text{sqrt}}}=\infty$};

    \end{tikzpicture}
}
    \caption{Movement distortions used when applying actions $\text{clip}_\lambda (s_W-s)$.}\label{f:movement_distortions}
\end{figure}

\subsubsection{Observations}\label{s:observation_types}

A canonical type of observation is when the position can be observed directly, i.e.,
$\mathcal{O}_{\text{PO}}(s, W)=s$. Here, we need to fix the optimum $s_W=s^*$, because it is impossible to infer
$s_W$ without observing $W$ (see also Section~\ref{app:logging_policy}). Roughly speaking, these are
scenarios where it is known where the optimum is, but not how to get there. We will evaluate these
scenarios in $d=2$ and $d=5$ dimensions.
Our motivation stems from scenarios where observations are drawn from optical
sensors and hence we test our method on different image generators
(Figure~\ref{fig:image_observations}). The first comes from active
alignments problems from camera assembly, where a lens objective has to be positioned 
relative to a sensor to obtain optimal optical performance~\cite{active_alignments_with_deep_learning}. 
Here, $s$ relates to the
position of the lens objective and $W$ to variances in the lenses of the objective and distortions in the movement dynamics. 
At each position $s$, light is sent through
the lens system creating an image
$\mathcal{O}_{\text{LP}}(s, W)$ on a sensor. The task is to position the objective with variances
$W$ precisely to an individual optimum~$s_W$ (Figure~\ref{f:active_positioning})
As some information about $W$ is contained in the image implicitly, it is possible to design
algorithms that leverage the image information to move towards $s_W$. We use the realistic generator
from~\cite{burkhardt2025relign} where light is sent in the form of a \emph{Siemens star}
producing images whose contrast and sharpness are sensitive to small misalignments.

We also run experiments in the \emph{Fetch Reach} environments~\cite{fetch}, where a robotic arm has
to reach a desired position $s_W$. Here, we use the vanilla environment
$\mathcal{O}_{\text{Fetch}}(s, W)=s-s_W$ where the distance to the target is observed. In
Section~\ref{app:fetch} we study the effect of shortcut augmentation for harder variants using image
observations $\mathcal{O}_{\text{FetchImg}}$ and reaching multiple goals subsequently from offline
data alone.

Our last image generator is the \emph{light tunnel} from~\cite{gamella2025chamber}, 
where light is sent through two polarizers whose angles dictate how 
it passes through to an optical sensor.
Here, each position $s$ of the polarizers filters out certain wavelengths of the light creating a
image~$\mathcal{I}(s)$ at the sensor. Here, $\mathcal{I}(s)$ does not depend on the context
$W$ but only on the relative difference of the angles of the polarizers, i.e. many
states lead to the same image. To add context, we sample
in each episode $s_W$ uniformly from the box $[0, 2\pi]^2$ and set
$\mathcal{O}_{\text{LT}}(s, W)=\mathcal{I}(s)-\mathcal{I}(s_W)$. In our experiments, we use the
decoder of the autoencoder trained on images from the real system provided in the data repository
of~\cite{gamella2025chamber}.

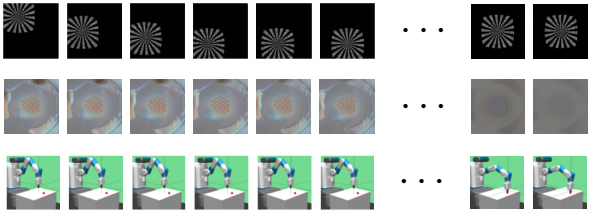
\begin{figure}[H]
    \centering
    \begin{tikzpicture}
        \node[] (lp_start) at (0,0) {\includegraphics[height=0.8cm]{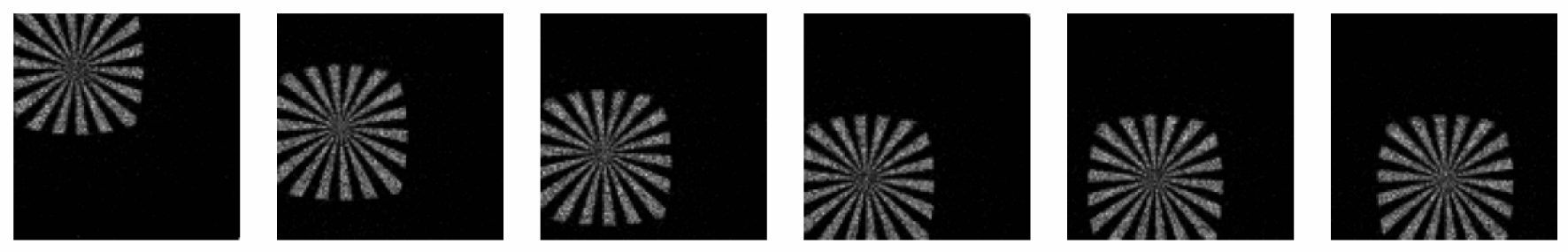}};
        \node[right=1cm of lp_start] {\includegraphics[height=0.8cm,trim={5cm 0 0 0},clip]{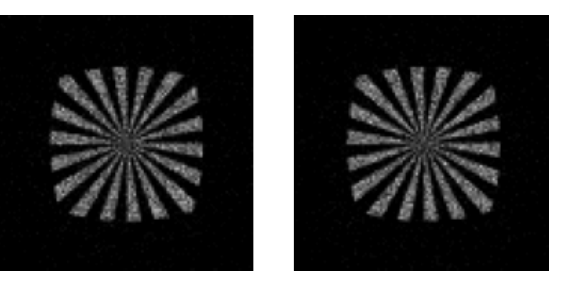}};

        \node[] (lt_start) at (0,-1.0) {\includegraphics[height=0.8cm]{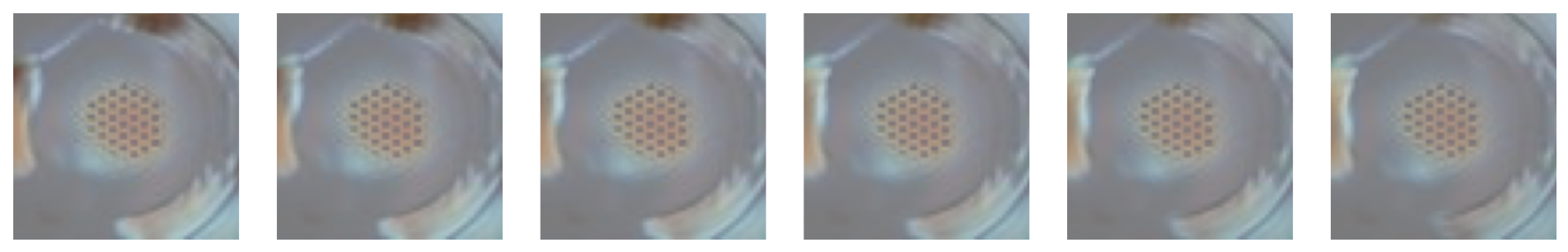}};
        \node[right=1cm of lt_start] {\includegraphics[height=0.8cm,trim={5cm 0 0 0},clip]{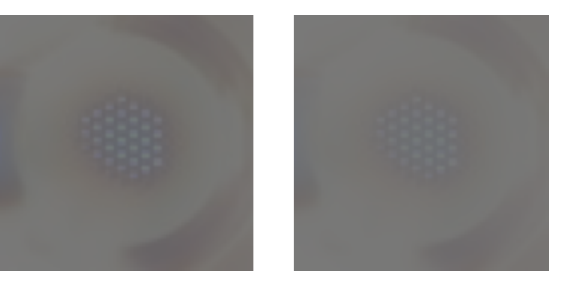}};

        \node[] (fr_start) at (0,-2.0) {\includegraphics[height=0.8cm]{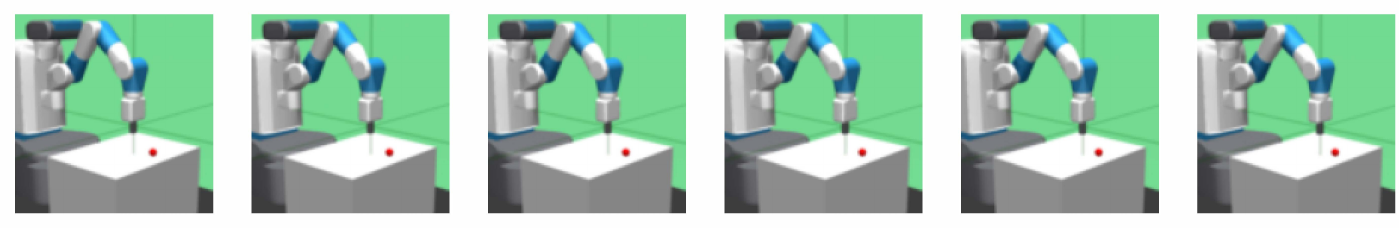}};
        \node[right=1cm of fr_start] {\includegraphics[height=0.8cm,trim={4cm 0 0 0},clip]{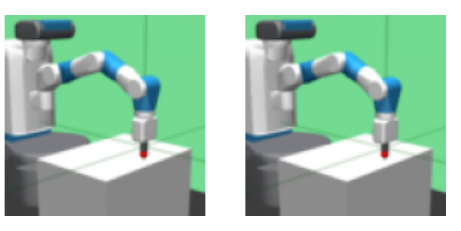}};

        \node[right=0 of lp_start, scale=1.5] {$\cdots$};
        \node[right=0 of lt_start, scale=1.5] {$\cdots$};
        \node[right=0 of fr_start, scale=1.5] {$\cdots$};
    \end{tikzpicture}
    \caption{
        Exemplary trajectories of $\cw{l}$ executed in
    $\mathcal{O}_{\text{LP}}$, $\mathcal{O}_{\text{LT}}$, and $\mathcal{O}_{\text{FetchImg}}$ (top to bottom).}\label{fig:image_observations}

\end{figure}

\subsubsection{Logging policies}\label{s:logging_policy}
In most offline RL benchmarks, logging policies are obtained by training online RL algorithms
partially or fully to obtain policies of different expertness~\citep{d4rl}. However,
in many real-world continuous-control settings, logging policies are hand-crafted, highly
structured, and systematically suboptimal routines. 
This is particularly common in active positioning tasks, where expert routines rely on relatively
simple mechanisms yet can be applied across a wide range of systems with only minimal adjustments.
A representative example are optical alignment procedures, in which system performance is improved
iteratively by sequentially adjusting individual degrees of freedom and evaluating a measured signal
such as coupling efficiency or spot quality~\citep{alignment_optical_systems,alignment_decam,
alignment_strategies}. Similar principles also apply to other positioning and manipulation tasks that
use coordinate-based or heuristic search strategies. These methods usually start with
rough movements and reduce the step size over time until the target is reached.
\citep[Section~3.1]{active_alignments_with_deep_learning}.
To study offline RL under such structured but imperfect data in a controlled and reproducible
manner, we require a logging policy that reliably reaches the target while producing trajectories
that are suboptimal in both direction and number of steps, and whose expertness can be varied
systematically. We distill these principles into a synthetic logging policy referred to as the
\emph{coordinate walk} $\cw{l}$.

Across all scenarios we study, successful control requires that relevant
displacement information — essentially $s - s_W$ — is inferable from the observation $\mathcal{O}(s,
W)$, since otherwise the task is not solvable. Even when $s - s_W$ is inferable, however, the task
may still be unsolvable without any information about the movement distortion $f$. This requirement
is discussed more formally in Appendix~\ref{app:logging_policy} and instantiated concretely in
Section~\ref{s:observation_types} for the various observation settings we study. To generate
reliable trajectories across these different observation settings and distortion regimes, we gave
the logging policy direct access to
$s - s_W$. Importantly, this does not make the task trivial, we simply assume the logging policy
already has a reliable way to infer the relevant information from $\mathcal{O}(s, W)$. 
Inspired by real-world logging policies as described above, we constructed a structured logging
policy that optimizes coordinate by coordinate.
That is, actions are chosen along coordinate axes
until the corresponding coordinate of $s$ matches that of $s_W$. Once all dimensions have been
traversed, the step size $l$ is reduced and the procedure is repeated, resulting in a reliable but
As a result, the logging policy can reach the target for the movement distortions we consider, but
it does so highly inefficiently, including overshoots, detours, and movement in the wrong direction.
In Section~\ref{sec:ablation_logging_policy}, we show that our method is not dependent on structured
logging policies. 

By varying the initial step size, the expertness of the logging policy can be adjusted (see
Figure~\ref{f:step_size_effect}). Figure~\ref{fig:example_trajectories} shows trajectories of the
coordinate walk executed under different movement distortions.
To model realistic hand-overs between logging policies and augmentors, we assume the internal state of the policy,
i.e. the current step size $l$ and dimensions already optimized, is
reset to the initial values once the policy is reset. To avoid making our mathematical framework
introduced in Section~\ref{sec:theory} too specific for these types of resets, we assume stateless
policies there. For most states, $V^{\pi_{l_2}}(s, W)\ge V^{\pi_{l_1}}(s,
W)$ for two step sizes $l_1<l_2$ holds true and thus Theorem~\ref{thm:shortcuts} holds in this
setting. In Section~\ref{app:contractions}, 
a detailed discussion on the contraction-property and LPE of $\cw{l}$ is given.

\subsection{Results}

Section~\ref{sec:lift} gives rise to two algorithms. First, a purely offline
one that takes a static dataset collected from some logging policy and trains an offline RL
algorithm with shortcut augmentations. In our experiments, we use CQL and denote this algorithm as
CQL-SC. Second, an iterative offline RL algorithm that collects data with an augmented logging
policy where CQL is trained on the collected data, called \ourmethod{}. If the subsequently trained
CQL also uses shortcuts, we denote this algorithm as \ourmethod{}-SC. By default, we use
Algorithm~\ref{alg:data_augmentation} with $p=0.6$, limit augmentations per trajectory to $20$.
In Section~\ref{sec:C_analyse}, we study in detail the sensitivity of our method to the choice of
the hyperparameter $C$. Larger values of $C$ are more restrictive in terms of which augmentations
are sampled. Although better policies can be obtained by tuning $C$, particularly when $L_f$ is
comparatively large like in $f_{\mathrm{regrot}}$, we set $C=0$ in all experiments to ensure a fair
comparison and to avoid introducing additional inductive biases into our method. A detailed
hyperparameter analysis is given in Section~\ref{sec:effect_of_hyperparameters}.

First, we analyze the effect of different augmentations while collecting data and the effect
of using shortcuts in the CQL training afterward. 
Beside naive augmentations as adding gaussian noise $\pi_\beta(o)+\eps$ or randomly scaling actions
$\pi_\beta(o)\cdot\eps$ with $\eps=2\cdot\exp(\eta), \eta \sim \mathcal{N}(0, \sigma)$, 
we also use uniformly sampled actions from $\cA$ and IORL-like augmentations
based on an uncertainty model as in~\cite{zhang2023active}.
We run these experiments in $(\mathcal{O}_{\text{PO}}, f_{\text{blend}})$ with step size $0.025$ in
$d=5$ dimensions, collected $3$ independent datasets consisting of $100$ trajectories each and
trained $3$ independent CQL policies on each of them. The LIFT augmentor is trained once after $50$
trajectories.
The averaged convergences to $s_W$ of the CQL policies, each evaluated on $20$ randomly drawn
contexts are shown in
Figure~\ref{fig:comparision_with_augmentations}. Once can see that independently whether shortcuts are
used in the training afterward, the best CQL policies is obtained when trained on the data
collected with \ourmethod{}. Moreover, we see that when training takes place with shortcuts, every
policy can be improved. This finding is underpinned when computing the dataset characteristics
introduced in~\cite{dataset-perspective} shown in Figure~\ref{fig:dataset_scores}. \ourmethod{}
creates trajectories having the highest average returns reproducing findings
in~\cite{dataset-perspective} that this correlates with CQL performance. On the other hand,
\ourmethod{} does not explore as well as other methods, showing a clear differentiation to
IORL that has been explicitly designed to explore well. However, high exploration comes at the price
of an impeded hand-off back to the logging policy, leading to low trajectory qualities for IORL and
random actions.

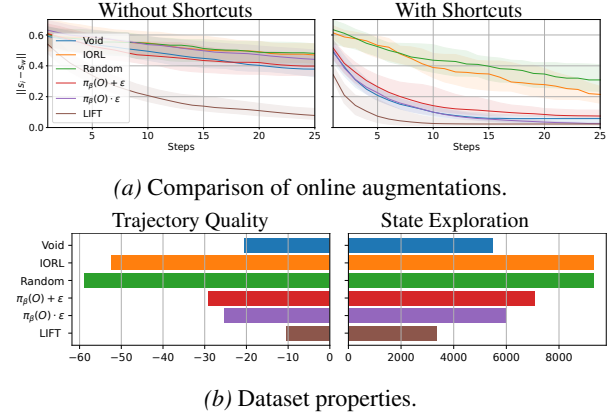
\begin{figure}[ht]
\subfloat[Comparison of online augmentations.]{
            \begin{tikzpicture}
        \node[] (lp_start) at (0,0) { \includegraphics[width=7.8cm]{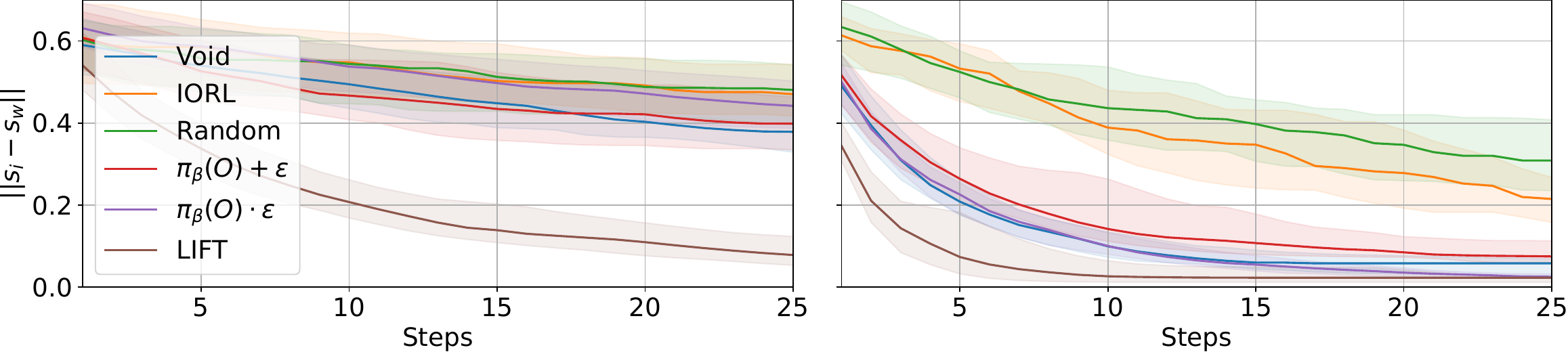}};
            \node[scale=0.8, align=center] at (-1.8, 1) {Without Shortcuts};
          \node[scale=0.8, align=center] at (1.95, 1) {With Shortcuts};
    \end{tikzpicture}
    \label{fig:comparision_with_augmentations} 
}
    \hfill
    \subfloat[Dataset properties.]{
        \begin{tikzpicture}
        \node[] (lp_start) at (0,0) { \includegraphics[width=7.8cm]{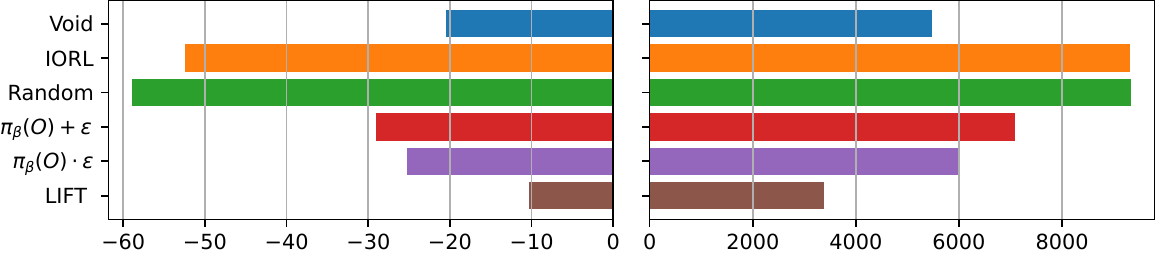} };
            \node[scale=0.8, align=center] at (-1.6, 1) {Trajectory Quality};
          \node[scale=0.8, align=center] at (1.9, 1) {State Exploration};
    \end{tikzpicture}
    \label{fig:dataset_scores}
    }
    \caption{Experiments in $(\mathcal{O}_{\text{PO}}, f_{\text{blend}})$ with $l=0.025$ and $d=5$.}
\end{figure}
 
In our second type of experiments, we evaluate how our methods compare under different
movement distortions and observation types. 
In $\mathcal{O}_{\text{PO}}$, algorithms collect a total of $n=100$ and $n=500$ trajectories for $d=2$ and $d=5$ respectively, where
the LIFT augmentor is trained once after $50$ and $100$ collected trajectories respectively. In
$\mathcal{O}_{\text{LP}}$, we collect $500$ trajectories and LIFT is trained once after $100$
episodes. In $\mathcal{O}_{\text{LT}}$, we collect only $100$ trajectories and LIFT is trained once after $50$
collected trajectories. Here, we additionally compare to SAC~\cite{sac} trained with a mixture of offline and online data as
done in warm-start RL that is restricted to the same number of trajectories as in our offline
datasets. Specifically, in a scenario with $n$ episodes, the replay buffer of SAC is
initialized with the same number of trajectories collected by the logging policy the \ourmethod{} augmentor
obtains in training, e.g. $m=50$ for $\mathcal{O}_{\text{LT}}$.
Moreover, we also compare to diffusion-based techniques, like GTA~\cite{gta} that generate
synthetic transitions and Diffusion-QL (DQL)~\cite{diffusion_ql} that learns a diffusion-based
policy.
Figure~\ref{fig:main_result} presents selected comparisons across the multiple scenarios 
and all comparisons can be found in Section~\ref{app:additional_experiments}.
In all tested environments, we see that CQL policies trained offline on data from \ourmethod{}
have
better performance than these trained on unaugmented data from the logging policy.
This effect fades a bit when adding shortcuts to the subsequent offline training: In most scenarios,
the performance of
\ourmethod{}-SC is better or equal than CQL-SC. This is, for instance, not the case
when using image data from $\mathcal{O}_{\text{LP}}$, where CQL training on data obtained
from \ourmethod{}-SC showed high variance.
Studying the effect of shortcuts in isolation, CQL-SC consistently outperforms CQL and
\ourmethod{}-SC consistently outperforms \ourmethod{}, making \ourmethod{}-SC the best of our
methods.
Comparing LIFT-SC to SAC with offline data, we see a clear picture:
SAC stays ahead in all low-dimensional cases for $\mathcal{O}_{\text{PO}}$, and
\ourmethod{}-SC outperforms SAC almost consistently over all movement dynamics and expert-levels of
the logging policy in
$\mathcal{O}_{\text{PO}}$ for $d=5$ (see Appendix~\ref{sec:sac_vs_lift}), as well as in image-based scenarios.
Interestingly, for $f_{\text{regrot}}$ where the contraction property is violated, augmentations
with shortcut fail, whereas in $f_{\text{sqrt}}$, where LPE does not hold, augmentations 
still help but the advantage over SAC is negligible.

 \begin{figure}[ht]
    \centering
    \begin{tikzpicture}
        \node[] at (0,0) {\includegraphics[width=\columnwidth]{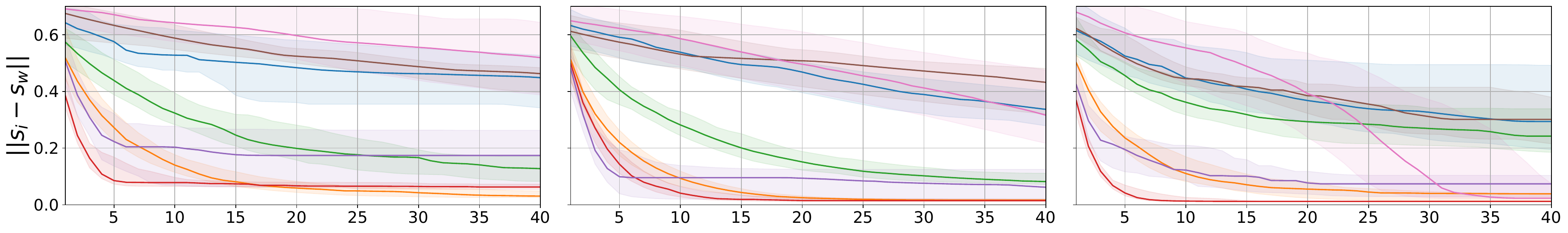}};
        \node[scale=0.5, align=center] at (-2.5, 0.8) {$\mathcal{O}_{\text{PO}}, f_{\text{regrot}}$\\$d=5, l=0.0125$};
        \node[scale=0.5, align=center] at (0, 0.8) {$\mathcal{O}_{\text{PO}}, f_{\text{scale}}$\\$d=5, l=0.025$};
        \node[scale=0.5, align=center] at (2.7, 0.8) {$\mathcal{O}_{\text{PO}}, f_{\text{rot}}$\\$d=5, l=0.05$};
        \node[] at (0,-1.8) {\includegraphics[width=\columnwidth]{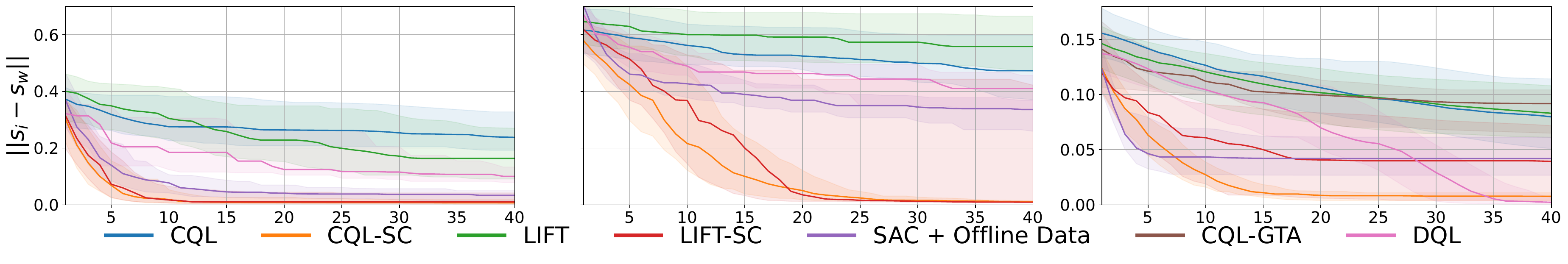}};
        \node[scale=0.5, align=center] at (-2.5, -0.9) {$\mathcal{O}_{\text{LT}}, f_{\text{blend}}$\\$d=2, l=0.025$};
        \node[scale=0.5, align=center] at (0, -0.9) {$\mathcal{O}_{\text{LP}}, f_{\text{blend}}$\\$d=5, l=0.025$};
        \node[scale=0.5, align=center] at (2.7, -0.9) {$\mathcal{O}_{\mathrm{Fetch}}, f_{\text{fetch}}$\\ $d=3, l=0.2$};
    \end{tikzpicture}
    \caption{Comparisons of our methods for selected scenarios.}
    \label{fig:main_result}
\end{figure}
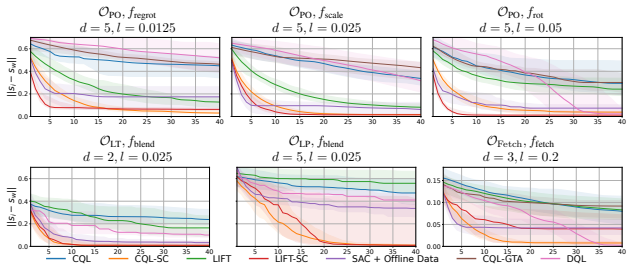

Finally, we analyse the effect of absence of structure in the logging policy on the performance of
the shortcut augmentation by injecting noise into the $\cw{l}$. The results are in presented in
Section~\ref{sec:ablation_logging_policy} and in the tested scenarios, we found that shortcut
augmentation consistently yields better policies, suggesting that benefits of shortcuts are not
limited to structured logging policies. 

\section{Discussion}

We demonstrate that shortcut augmentations can consistently improve the effectiveness of offline
RL in active positioning problems in both, theoretical and experimental
validations. In particular, we find that augmentations provide the largest gains in complex
scenarios with higher action dimensionality or partial observability, where plain offline RL often
fails. This suggests that exploiting task structure to expand data coverage is a promising
alternative to relying solely on behavior regularization. Compared to warm-start RL, \ourmethod{}
offers a more data-efficient way to leverage suboptimal expert routines: by selectively taking
shortcuts suggested by an off-policy learner, we improve dataset quality without requiring extensive
online fine-tuning. 
Nevertheless, our approach has limitations. Shortcut validity depends on assumptions about the
distortion function and value function regularity, which may not hold in all real-world positioning
systems. Moreover, our experiments are limited to semi-realistic simulators; future work should
validate these methods on physical platforms, especially in robotic alignment tasks. 
Another open question is how to combine shortcut augmentation with model-based methods or world models to further
improve sample efficiency.
We believe that the principles underlying \ourmethod{} are broadly applicable beyond the scenarios
studied in in this work where expert routines exist but are suboptimal. We hope this work encourages a
more systematic treatment of data augmentation strategies for offline RL in structured industrial
tasks.

\section*{Acknowledgments}
This research was funded by the German Federal Ministry of Research, Technology and Space (BMFTR)
under grant number 13FH605KX2. TW is funded by the \emph{Hightech Agenda Bavaria}. We thank our
colleagues Michael Layh and Martin Wenzel for helpful discussions and feedback on the manuscript.

\section*{Impact Statement}
This paper presents work whose goal is to advance the field
of Machine Learning. There are many potential societal
consequences of our work, none which we feel must be
specifically highlighted here.

\bibliography{lift}
\bibliographystyle{icml2026}

\newpage
\appendix
\onecolumn
\section{Proofs for Section~\ref{sec:theory}}\label{s:proofs}

\begin{prop}\label{prop:augmented_is_better}
Let $\pi_\beta$ and $a_\theta$ be two policies and $o=O(s, W)$, then $J(\pi_{\text{aug}})\ge J(\pi_\beta)$ with $\pi_{\text{aug}}$
defined as follows:
\begin{equation*}
    \pi_{\textnormal{aug}}(o) :=
\begin{cases}
    a_\theta(o) \quad\textnormal{if } a_\theta(o)\textnormal{ is a $\pi_\beta$-shortcut at } (s,
    W)\\
\pi_\beta(o) \quad\textnormal{otherwise}
\end{cases}.
\end{equation*}
\end{prop}
\begin{proof}[Proof of Proposition~\ref{prop:augmented_is_better}]

We denote $\pi_\beta$ simply by $\pi$ in the following.
It suffices to show that the statement holds if augmentation only is applied at one single state
$(\tilde s, W)$ as we than can apply the statement repeatedly. That is, there
exists an action $a$ that satisfies:
$$\gamma\cdot V^{\pi}(f(\tilde s, a, W), W) - \|f(\tilde s, a, W)-s_W\| \ge V^\pi(\tilde s, W) $$

Let $\pi_a$ be the policy that uses $a$ at $\tilde s$ and on all other states coincides with $\pi$. First,
we show that $J(\pi_a)\ge J(\pi)$. It suffices to show that $V^{\pi_a}(s)\ge V^{\pi}(s)$ for all $s\in S$.
Let $(s, W)$ be an initial state. If the trajectory of $\pi$ does not traverse $\tilde s$, then $V^{\pi_a}(s)=
V^{\pi}(s)$. Assume differently that the trajectory visits $\tilde s$ at the $t$-th step.
Then, the trajectory starting at $s$ follows $\pi$ till $\tilde s$, then chooses the shortcut $a$, and then follows
$\pi$ from $s'=f(\tilde s,a, W)$. The value for this trajectory is:

$$V^{\pi_a}(s, W)=V^\pi(s, W) - \gamma^t\cdot V^\pi(\tilde s, W) - \gamma^t\|s'-s_W\|+\gamma^{t+1}
V^\pi(s', W).$$
From the assumption of $(\tilde s, a)$, we have

$$\gamma^t\cdot (-V^\pi(\tilde s, W) - \|s'-s_W\|+\gamma\cdot V^\pi(s', W))\ge 0$$
and hence $V^{\pi_a}(s, W)\ge V^\pi(s, W)$.
\end{proof}

\begin{lemma}\label{lemma:lower_bound}
    Let $\pi$ be distance-improving, then $(1-\gamma)V^\pi(s, W)\ge -\|s-s_W\|$ for all
    $(s, W)$.
\end{lemma}
\begin{proof}
Let $(s_0, W), (s_1, W), \ldots, (s_k, W)$ be a trajectory of $\pi$ starting at $s=s_0$, then
$$V^\pi(s, W)=-\sum_{i=1}^k\gamma^{i-1}\|s_i-s_W\|\ge -\|s-s_W\|\sum_{i=0}^{k-1}\gamma^i =
-\|s-s_W\|\cdot\frac{1-\gamma^k}{1-\gamma}$$
where we have used that $\pi$ is distance improving in every step.
Finally, $(1-\gamma)V^\pi(s, W)\ge -\|s-s_W\|(1-\gamma^k)\ge -\|s-s_W\|$.
\end{proof}

\begin{proof}[Proof of Proposition~\ref{prop:augmentations}]

    Assume that $\tau=(s_0,\ldots, s_k)$ is the sub-trajectory of $\pi$ starting at $s=s_0$ and
    ending at $s'=s_k$. We prove the statement via induction on $k$. Note that since $s'\neq s$, we
    have $k\ge 1$. Let $k=1$, then
    $$V^\pi(s, W)=-\|s_1-s_W\| + \gamma\cdot V^\pi(s', W)$$
    and the claim holds. Now, assume the statement holds from $s_1$ to $s_k=s'$, then
    $$\gamma V^\pi(s', W)-V^\pi(s_1, W)\ge \|s'-s_W\|$$
    by the induction hypothesis. Furthermore, we have
        \begin{align*}
            \gamma V^\pi(s', W)-V^\pi(s ,W)&= \gamma V^\pi(s', W)-V^\pi(s_1, W)+V^\pi(s_1, W)-V^\pi(s, W)\\
                            &\ge \|s'-s_W\| + V^\pi(s_1, W)-V^\pi(s, W)\\
                            &= \|s'-s_W\| + V^\pi(s_1, W)-(-\|s_1-s_W\|+\gamma V^\pi(s_1, W))\\
                            &= \|s'-s_W\| + (1-\gamma)V^\pi(s_1, W) + \|s_1-s_W\|
        \end{align*}
Using Lemma~\ref{lemma:lower_bound}, we have $(1-\gamma)V^\pi(s_1, W) + \|s_1-s_W\|\ge 0$ and the claim
follows.
\end{proof}

\begin{proof}[Proof of Proposition~\ref{prop:shortcuts_for_linear_dynamics}]
    Since Proposition~\ref{prop:augmentations} gives that $\gamma V^\pi(s_j, W)-V^\pi(s_i,
    W)\ge\|s_j-s_W\|$, it is left to prove that $f(s_i, a, W)=s_j$. We have

    $$f(s_i, a, W)=s_i+W\cdot\sum_{k=i}^{j-1}a_i=s_i+W\cdot a_i +W\cdot a_{i+1}+\ldots + W\cdot
    a_{j-1}.$$
    Let $s_{i+1},\ldots,s_{j-2}$ be the intermediate states, i.e. $s_k=f(s_{k-1}, a_{k-1}, W)$, then
    replacing $s_k=s_k-1+W\cdot a_{k-1}$ in the equation above from $k=i$ to $k=j-1$ gives the
    claim.
\end{proof}

\begin{proof}[Proof of Proposition~\ref{prop:lpe_for_bounded_transform}]
Let $a_0,\ldots, a_{k-1}$ a chain of actions and set $A=\sum_{i=0}^{k-1}=a_i$, $(s_0, W)$ an initial
state and set $s_i=f(s_{i-1}, a_{i-1}, W)$. Recursively unraveling the definition of $f$ yields
$$s_k=s_0+\sum_{i=0}g(s_i, W)\cdot a_i$$
and consequently
\begin{align*}f(s_0,A,W)-s_k
&= g(s_0, W)\sum_{i=0}^{k-1}a_i-\sum_{i=0}^{k-1}g(s_i,W))a_i\\
&= \sum_{i=0}^{k-1}\bigl(g(s_0, W)-g(s_i, W)\bigr)a_i.
\end{align*}
Taking norms and using the induced matrix norm on $\RR^{m\times d}$ gives
$$
\bigl\|f(s_0,A,W)-s_k\bigr\|
\le \sum_{i=0}^{k-1}\bigl\|g(s_0, W)-g(s_i, W)\bigr\|\cdot\|a_i\|.
$$
By the assumption on $g$, we have
$$\|g(s_0, W)-g(s_i,W)\|\le\|g(s_0, W)\|+\|g(s_i,
W)\|\le 2\cdot\sup_{\cS\times\cW}\|g\|$$ independently of the actions
for all $i$ and the claim follows.
\end{proof}

\begin{proof}[Proof of Theorem~\ref{thm:shortcuts}]
For brevity, we omit $W$ in the notation of the value function.
We have to show that $\gamma V^\pi(f(s_i, a, W)) - V^\pi(s_i)\ge \|f(s_i, a, W)-s_W\|$. Because $f$
has linear-placement errors, it follows directly from Definition~\ref{def:linear_placement_errors}
that $\|f(s_i, a, W) - s_j\|\le L_f\cdot \sum_{k=i}^{j-1}\|a_k\|$ and
thus
$$\|f(s_i, a, W)-s_W\|=\|f(s_i, a, W) - s_j+s_j - s_W\|\le L_f\cdot \sum_{k=i}^{j-1}\|a_k\| +
\|s_j-s_W\|.$$
On the other hand, using the Lipschitz continuity of $V^\pi$, we get
\begin{align*}
    \gamma V^\pi(f(s_i, a, W)) - V^\pi(s_i)&\ge \gamma\cdot (V^\pi(s_j) - L_V\cdot \|f(s_i, a, W)-s_j\|) - V^\pi(s_i)\\
    &\ge \gamma\cdot V^\pi(s_j) - V^\pi(s_i) - \gamma\cdot L_V\cdot  L_f\cdot \sum_{k=i}^{j-1}\|a_k\|
\end{align*}
Now, as the inequality from the theorem statement holds, we have 
$$\gamma\cdot V^\pi(s_j) - V^\pi(s_i) \ge (\gamma\cdot L_V +1)\cdot L_f\cdot
\sum_{k=i}^{j-1}\|a_k\| +\|s_j-s_W\|$$
and plugging this into the upper equation gives the claim.
\end{proof}

\begin{prop}\label{prop:lipschitz_value_function}
    Let $\pi$ be an $f$-contraction. Then $V^\pi$ is $\frac{1}{1-\gamma}$-Lipschitz continuous in
    the states.
\end{prop}
\begin{proof}
    Define $L=\frac{1}{1-\gamma}$ and let $(s,W)$ and $(s', W)$ be two states. 
    We prove via
    induction over the combined number of steps $k$ needed to reach the optimality region
    around $s_W$ starting at $s$ and $s'$ that
    $$|V^\pi(s, W)-V^\pi(s', W)|\le L\cdot \|s-s'\|.$$
    If $k=0$, then $s$ and $s'$ are both within the optimality region, i.e. $\|s-s_W\|\le \theta$
    and $\|s'-s_W\|\le \theta$, then $V^\pi(s, W)=V^\pi(s', W)=0$ and the claim holds.
    Now, let $o=O(s, W)$ and $o'=O(s', W)$ be the observations at $s$ and $s'$ and
    $s_1=f(s,\pi(o), W)$ and $s_1'=f(s', \pi(o'), W)$ be the next states after one step of
    $\pi$. Particularly, the induction hypothesis holds for $s_1$ and $s_1'$, i.e.
    $|V^\pi(s_1,W)-V^\pi(s_1', W)|\le L\cdot \|s_1-s_1'\|$. Since
    $V^\pi(s)=-\|s_1-s_W\|+\gamma V^\pi(s, W)$ and $V^\pi(s')=-\|s_1'-s_W\|+\gamma V^\pi(s', W)$, we have

    \begin{align*}
        |V^\pi(s) - V^\pi(s')|&=|\gamma\cdot V^\pi(s_1, W) - \gamma\cdot V^\pi(s'_1, W) - \|s_1-s_W\| + \|s_1'-s_W\||\\
                              &\le \gamma\cdot |V^\pi(s_1, W) - V^\pi(s'_1, W)| + |\|s_1-s_W\| - \|s_1'-s_W\||\\
                              &\le \gamma\cdot L\cdot \|s_1-s_1'\| + \|s_1-s_1'\|\\
                              &\le (\gamma\cdot L+1)\cdot \|s_1-s_1'\|\\
                              &= L\cdot \|s_1-s_1'\|
    \end{align*}
    where the last equation is due to $L=\frac{1}{1-\gamma}$. Finally, because $\pi$ is an
    $f$-contraction, we have $\|s_1-s_1'\|=\|f(s, \pi(o), W)-f(s', \pi(o'), W)\|\le\|s-s'\|$ and the
    claim follows.
\end{proof}

\begin{proof}[Proof of Corollary~\ref{cor:shortcuts}]
    Because $\pi$ is an $f$-contraction, $V^\pi$ is $\frac{1}{1-\gamma}$-Lipschitz continuous by
    Proposition~\ref{prop:lipschitz_value_function}. Plugging $L_V=\frac{1}{1-\gamma}$ into Theorem~\ref{thm:shortcuts} gives the claim.
\end{proof}

\section{Movement distortion functions}\label{app:properties_of_distortions}

In this section, we formally define the different movement distortions $f$ we consider in our experiments.
The first set of distortions are linear distortions of the form 
$f(s, a, W)=s+W\cdot a$ with $W\in\RR^{d\times d}$ a distortion matrix, more specific, we use
$$f_{\text{blend}}(s, a, W)=s+(I_{d\times d}+W)\cdot a,\quad W\sim\mathcal{N}_{d\times d}(0,\sigma)$$
For $W\in\RR$ a scalar, let
$R_W=\begin{pmatrix}\cos(W) & -\sin(W) \\ \sin(W) &
\cos(W)\end{pmatrix}$ be a two-dimensional rotation matrix. We rise this to a high-dimensional
rotation matrix where adjacent dimensions are rotated, i.e.,
$$\text{Rot}_W=\text{diag}(R_W,\ldots, R_W)\in\mathbb{R}^{d\times d}$$
where $\text{diag}(A_1,\ldots, A_k)$ is the block-diagonal matrix with blocks $A_1,\ldots, A_k$ on the
diagonal.

$$f_{\text{rot}}(s, a, W)=s+\text{Rot}_W \cdot a,\quad W\sim \mathcal{N}(0,\sigma)$$

The next distortion function is a scaling-based one which does not depend on a latent context $W$:
$$f_{\text{scale}}(s, a, W)=s+\text{clip}_{C, \lambda}\left(\|s-s_W\|\right)\cdot a$$
with some constant $0<C<\lambda$ to ensure that the steps are not to small so that the optimum can be reached in finitely many steps.

The next set of distortions is again a rotation-based one, but one where the rotation matrix depends
on the region. For that, we assume the position space $\cP$ is decomposed into $c$-many
non-overlapping subsets
$\cP_1, \ldots, \cP_c$ such that $\cup_{i=1}^c \cP_i=\cP$. Then

$$f_{\text{regrot}}(s, a, W)=s+\sum_{i=1}^c \mathbf{1}_{s\in\cP_i}\cdot \text{Rot}_{W_i}\cdot a,\quad
W\in\mathcal{N}_c(\mu,\sigma), \mu\in\RR^c$$
As $\cP_i\cap\cP_j=\emptyset$ for $i\ne j$, only one rotation matrix is active at a time, depending
on the state.

In our experiments, we used $c=4$ and
divided $\cP$ into four sets depending on in which quadrant of $\RR^2$ the first two dimensions
reside. Moreover, we set $\mu=(-0.3, 0.6, -0.3, 0.6)$.

The next distortion is one where a non-linear offset is added which depends on both, the state and the action:
$$f_{\text{sin}}(s, a, W)=s+a+W\cdot\sin(s)\circ \cos(s)\cdot\|a\|,\quad W\sim\mathcal{U}(0, \sigma)$$
where $\sin$ and $\cos$ are applied component-wise and $\circ$ denote the element-wise
multiplication. Finally, we consider a distortion function that does not have linear placement errors:
$$f_{\text{sqrt}}(s, a, W)=s+(I_{d\times d}+W)\cdot\sqrt{\|a\|}\cdot a,\quad
W\sim\mathcal{N}_{d\times d}(0,\sigma).$$

\subsection{Linear placement-errors}

We begin by proving a stronger conditions, which is easier to check and implies LPE:

\begin{prop}\label{prop:linear_placement_erros}
    Let $f$ be a distortion function and assume there exists a constant $L_f$
    such that for all states $(s,W)$ and actions $a, a'\in\cA$ 
$$\|f(s, a+a', W) - f(f(s, a, W), a', W)\|\le L_f\cdot \|a\|$$
Then $f$ has LPE with constant $L_f$.
\end{prop}
\begin{proof}
    For $i\in\{0,\dots,k\}$, define the tail sums $\tilde a_i := \sum_{j=i}^{k-1} a_j$
and the states $\tilde s_i := f(s_i, \tilde a_i, W)$.
By definition $\tilde s_0 = f(s_0,a_0+\ldots+ a_{k-1},W)$ and,
since $\tilde a_k=0$ and $f(s,0,W)=s$, we also have $\tilde s_k = s_k$. Thus, we have to prove that
$\|\tilde s_0 - \tilde s_k\| \le L_f \sum_{i=0}^{k-1} \|a_i\|$.
Now, for any $i\in \{0,\dots,k-1\}$ we have
$$\|\tilde s_i - \tilde s_{i+1}\|
  = \|f(s_i,a_i+\tilde a_{i+1},W) - f(s_{i+1},\tilde a_{i+1},W)\|
  \;\le\; L_f \|a_i\|.$$
because of the assumptions on $f$ from the statement of the proposition.
Summing these inequalities and applying the triangle inequality yields
$$\|\tilde s_0 - \tilde s_k\|
  \le\sum_{i=0}^{k-1} \|\tilde s_i - \tilde s_{i+1}\|
  \le L_f \sum_{i=0}^{k-1} \|a_i\|.$$
\end{proof}

LPE and the proposition of Proposition~\ref{prop:linear_placement_erros} are not equivalent: Consider
$f(s, a)=s+\text{sign}(s)\cdot a$. Then its easy to show that $f$ has linear-placement errors with
$L_f=2$, but it does not have the property from Proposition~\ref{prop:linear_placement_erros}.

\begin{prop}
    The distortion $f_{\mathrm{blend}}$ has LPE with $L_{f_{\text{blend}}}=0$.
\end{prop}

\begin{proof}
    Straight-forward application of Proposition~\ref{prop:linear_placement_erros}.
\end{proof}

\begin{prop}
    The distortion $f_{\mathrm{rot}}$ has LPE with $L_{f_{\text{rot}}}=0$.
\end{prop}

\begin{proof}
    Straight-forward application of Proposition~\ref{prop:linear_placement_erros}.
\end{proof}

\begin{prop} The distortion $f_{\mathrm{scale}}$ has LPE with $L_{f_{\text{scale}}}=2\cdot\lambda$.
\end{prop}
\begin{proof}
We write $f_{\mathrm{scale}}(s,a,W)=s+g(s, W)\cdot a$ with
$g(s, W)=\mathrm{clip}_{C,\lambda}(\|s-s_W\|)\cdot I_d$ with $I_d$ the identity function of
$\RR^{d\times d}$. Clearly $g$ is bounded and we have $\sup_{\cS\times\cW}\|g\|=\lambda$ and the claim follows by an application of
Proposition~\ref{prop:lpe_for_bounded_transform}.
\end{proof}

\begin{prop}
The distortion $f_{\text{regrot}}$ has LPE with $L_{f_{\text{regrot}}}=2$.
\end{prop}
\begin{proof}
We write
$f_{\mathrm{regrot}}(s,a,W)=s+g(s, W)\cdot a$ with $g(s,W)=\mathrm{Rot}_{W_i}$ whenever $s\in\mathcal
\cP_i$, where $\cP_1,\ldots,\cP_c$ are the partitions of $\cS$ from
Section~\ref{s:movement_distortion}. For every state $(s, W)$, $g(s, W)$ is a rotation matrix and
thus $\|g(s, W)\|=1$ and $g$ statisfies the the claim follows from 
Proposition~\ref{prop:lpe_for_bounded_transform}.
\end{proof}

\begin{prop}
The distortion $f_{\mathrm{sin}}$ 
has LPE with $L_{f_{\text{sin}}}=\sqrt{d}\sigma$.
\end{prop}

\begin{proof}
Let $f_{\mathrm{sin}}(s,a,W)=s+a+g(s)\cdot\|a\|$ with
$g(s, W):=W\cdot\sin(s)\odot\cos(s)$. 
Although we cannot apply Proposition~\ref{prop:lpe_for_bounded_transform} as $f_{\mathrm{sin}}$ has
not the desired form, we can follow a similar strategy. First, we observe that $g$ is bounded:
$$\|g(s, W)\|=|W|\cdot\sqrt{\sum_{i=1}^d\sin(s_i)^2\cdot\cos(s_i)^2}\le \sigma \sqrt{d}$$
because $W\sim\mathcal{U}(0,\sigma)$.
Let $a_0,\ldots,a_{k-1}$ be a chain of actions and set $A=\sum_{i=1}^{k-1}a_i$ and $s_i=f(s_{i-1},
a_{i-1}, W)$, then
$$
f_{\mathrm{sin}}(s_0,A,W)-s_k
= A+g(s_0, W)\|A\|
   -\sum_{i=0}^{k-1}\Bigl(a_i+g(s_i, W)\|a_i\|\Bigr)
= g(s_0, W)\|A\|-\sum_{i=0}^{k-1}g(s_i, W)\|a_i\|$$
and thus:
$$
\|f_{\mathrm{sin}}(s_0,A,W)-s_k\|\le\|g(s_0, W)\|A\|+\sum_{i=0}^{k-1}\|g(s_i,
W)\|a_i\|\le\sigma\sqrt{d}\sum_{i=0}^{k-1}\|a_i\|
$$
because $\|A\|\le \sum_{i=0}^{k-1}\|a_i\|$ by the triangle inequality.
\end{proof}

Next, we show that $f_{\mathrm{sqrt}}$ is not LPE:

\begin{prop}
The distortion $f_{\mathrm{sqrt}}$ 
does not have LPE.
\end{prop}
\begin{proof}
Let $v\in\RR^d$ be a unit vector and let $a_0=a_1=c\cdot v$ with $c\le \lambda$. 
Let $(0, 0)\in\RR^d\times\RR^{d\times d}$ be an initial state, then
$s_1=f_{\mathrm{sqrt}}(0, a_0, 0)=\sqrt{c}\cdot c\cdot v$
and
$s_2=f_{\mathrm{sqrt}}(s_1, a_1, 0)=2\sqrt{c}\cdot c\cdot v$. Moreover, we have
$f(s_0, a_0+a_1, 0)=f(0, 2\cdot c\cdot v, 0)=2\sqrt{2c}\cdot c\cdot v$ and hence
$$\|f(s_0, a_0+a_1, 0)-s_2\|=(2\sqrt{2}-2)\cdot \sqrt{c}\cdot c.$$
which cannot be bounded by $L_f\cdot (\|a_0\|+\|a_1\|)=2\cdot L_f\cdot c$ for any constant $L_f$.
\end{proof}

\subsection{Contractions and Lipschitz-continuity in real-world applications}\label{app:contractions}

We do not expect that policies and distortions from real-world applications satisfy the rigorous
mathematical assumptions stated in Section~\ref{sec:theory}.  Pedantically, even simple modeling
choices already break global smoothness: for instance, having $\cA=B_\lambda(0)$ with $\cA$ a strict
subset of $S$, combined with an optimality region defined by a threshold $\theta$, induces
discontinuities in the value function. The same holds for the coordinate walk policy in
Section~\ref{s:logging_policy}, where a fixed step length produces value functions with sharp
discontinuities, as shown in Figure~\ref{f:value_functions}.

Nevertheless, global mathematical rigor is not required to detect local shortcuts in real
trajectories. A striking example is the coordinate walk under $f_{\text{regrot}}$: since different
rotations apply in different regions, the policy is not an $f$-contraction globally, because nearby
states $s$ and $s’$ lying in different regions $\cP_i$ and $\cP_j$ may be rotated in different
directions (Figure~\ref{fig:contraction_failure}). Yet, for states within same region where the
coordinate walk applies same actions, the contraction property is preserved
(Figure~\ref{fig:contraction_sucess}). This illustrates that shortcut identification relies less on
global guarantees and more on local structure along trajectory segments.

Informally speaking, it suffices that the value function does not change too abruptly for small
misplacements, so that local improvements can be exploited as shortcuts. In practice, this condition
is often met: physical systems typically exhibit continuity over small ranges of motion, even if
discontinuities or non-contractive behavior emerge globally. Hence, while our theoretical
assumptions provide clean guarantees, the underlying ideas remain applicable well beyond the
idealized setting as demonstrated by our experiments in Section~\ref{sec:experiments}.

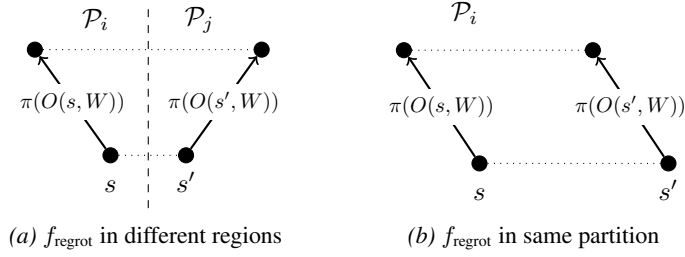
\begin{figure}[ht]
    \centering
    \subfloat[$f_{\text{regrot}}$ in different regions]{\label{fig:contraction_failure}
    \begin{tikzpicture}

        \node[draw, inner sep=2pt, circle, fill=black] at (-0.5,0) (s0) {};
        \node[draw, inner sep=2pt, circle, fill=black] at (-1.5,1.4) (s2) {};
        \node[draw, inner sep=2pt, circle, fill=black] at (0.5,0.0) (s1) {};
        \node[draw, inner sep=2pt, circle, fill=black] at (1.5,1.4) (s3) {};

        \node[below=0.5cm of s0,anchor=south] () {$s$};
        \node[below=0.5cm of s1,anchor=south] () {$s'$};

        \draw[->,thick] (s0) to node[midway, fill=white, scale=0.8] {$\pi(O(s, W))$} (s2);
        \draw[->,thick] (s1) to node[midway, fill=white, scale=0.8] {$\pi(O(s', W))$} (s3);

        \node[anchor=west] at (-1,1.8) () {$\cP_i$};
        \node[anchor=east] at (1,1.8) () {$\cP_j$};
        \draw[dashed, thin] (0, -0.7) to (0, 2);

        \draw[dotted] (s0) to (s1);
        \draw[dotted] (s2) to (s3);

    \end{tikzpicture}
}
\hspace{1cm}
    \subfloat[$f_{\text{regrot}}$ in same partition]{\label{fig:contraction_sucess}
    \begin{tikzpicture}

        \node[draw, inner sep=2pt, circle, fill=black] at (-0.5,0) (s0) {};
        \node[draw, inner sep=2pt, circle, fill=black] at (2.0,0.0) (s1) {};
        \node[draw, inner sep=2pt, circle, fill=black] at (-1.5,1.5) (s2) {};
        \node[draw, inner sep=2pt, circle, fill=black] at (1,1.5) (s3) {};

        \node[below=0.5cm of s0,anchor=south] () {$s$};
        \node[below=0.5cm of s1,anchor=south] () {$s'$};

        \draw[->,thick] (s0) to node[midway, fill=white, scale=0.8] {$\pi(O(s, W))$} (s2);
        \draw[->,thick] (s1) to node[midway, fill=white, scale=0.8] {$\pi(O(s', W))$} (s3);
        \node[anchor=west] at (-1,2) () {$\cP_i$};

        \draw[dotted] (s0) to (s1);
        \draw[dotted] (s2) to (s3);

    \end{tikzpicture}
}
    \caption{In $f_{\text{regrot}}$, starting at two close-by states $s$ and $s'$ in different
    regions $\cP_1$ and $\cP_2$ can increase the distance between subsequent states as opposed
rotation matrices apply.}
\end{figure}

\begin{figure}[ht]
    \centering
    \includegraphics[width=\textwidth]{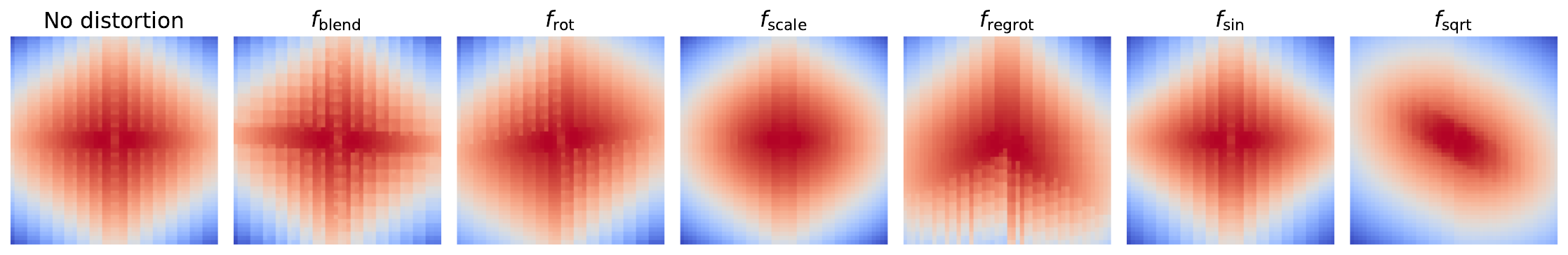}
    \caption{Value functions $V^\pi(\cdot, W)$ of coordinate walk for a random but fixed context~$W$
    each.}\label{f:value_functions}
\end{figure}

\section{Additional details for structured logging policies}\label{app:logging_policy}

This section provides additional details on the coordinate walk policy $\cw{l}$ introduced in
Section~\ref{s:logging_policy} and some insights on optimal policies for active positioning tasks. Figure~\ref{f:step_size_effect} illustrates how the step size $l$
impacts the expertness of the coordinate walk policy in terms of the average number of steps to
reach $s_W$. As designed, smaller step sizes lead to more expert behavior.

\begin{figure}[H]
\centering
    \includegraphics[height=5cm]{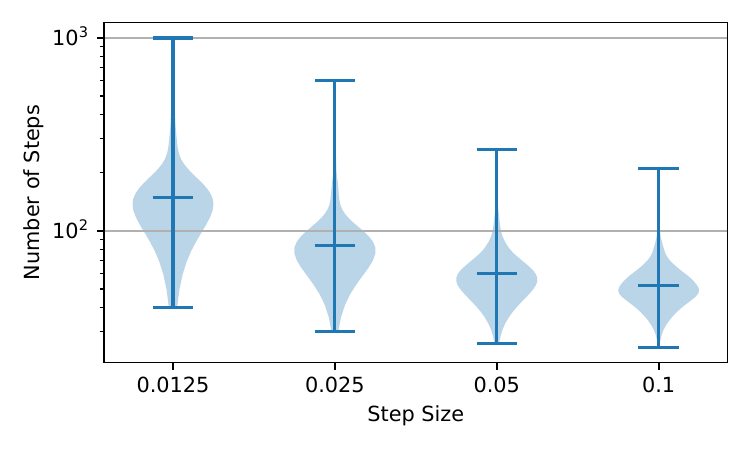}
    \caption{Expertness of $\cw{l}$.}\label{f:step_size_effect}
\end{figure}

The coordinate walk policy interacts quite differently with the various movement distortions.
Figure~\ref{fig:example_trajectories} shows example trajectories of the coordinate policy for
different movement distortions. There, we also compare to a direct policy
$\pi_{\text{direct}}$ that always takes the largest possible step $\mathrm{clip}_\lambda (s_W-s)$
towards the goal.
\begin{figure}[H]
    \centering

    \subfloat[Direct policy in $f_{\text{blend}}$]{
    \includegraphics[width=0.45\textwidth]{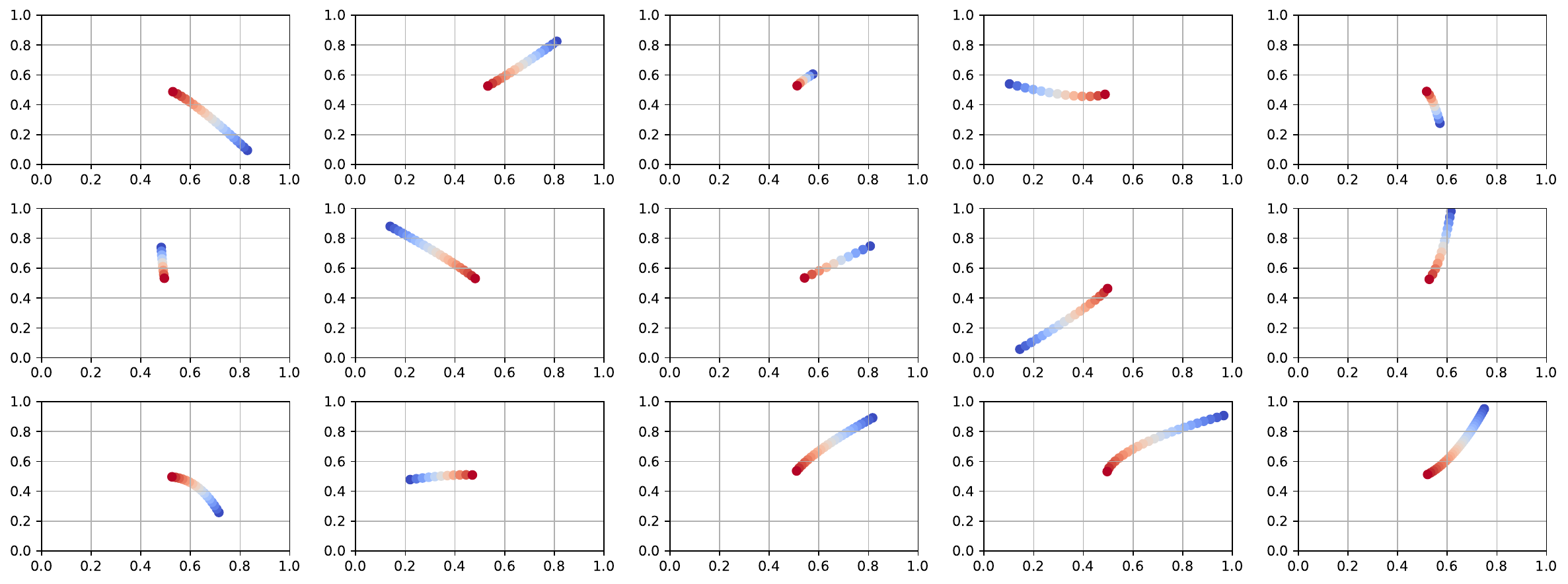}
    }
    \hfill
    \subfloat[Coordinate walk in $f_{\text{blend}}$]{
    \includegraphics[width=0.45\textwidth]{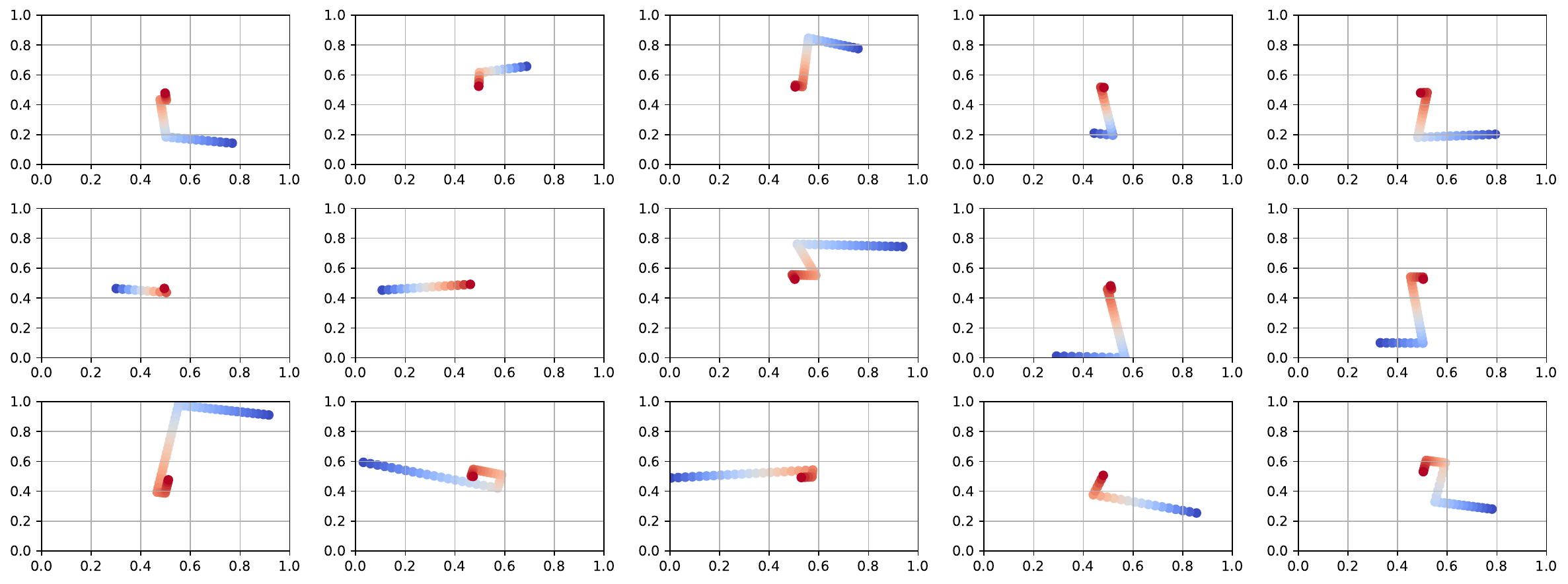}
    }

    \subfloat[Direct policy in $f_{\text{scale}}$]{
    \includegraphics[width=0.45\textwidth]{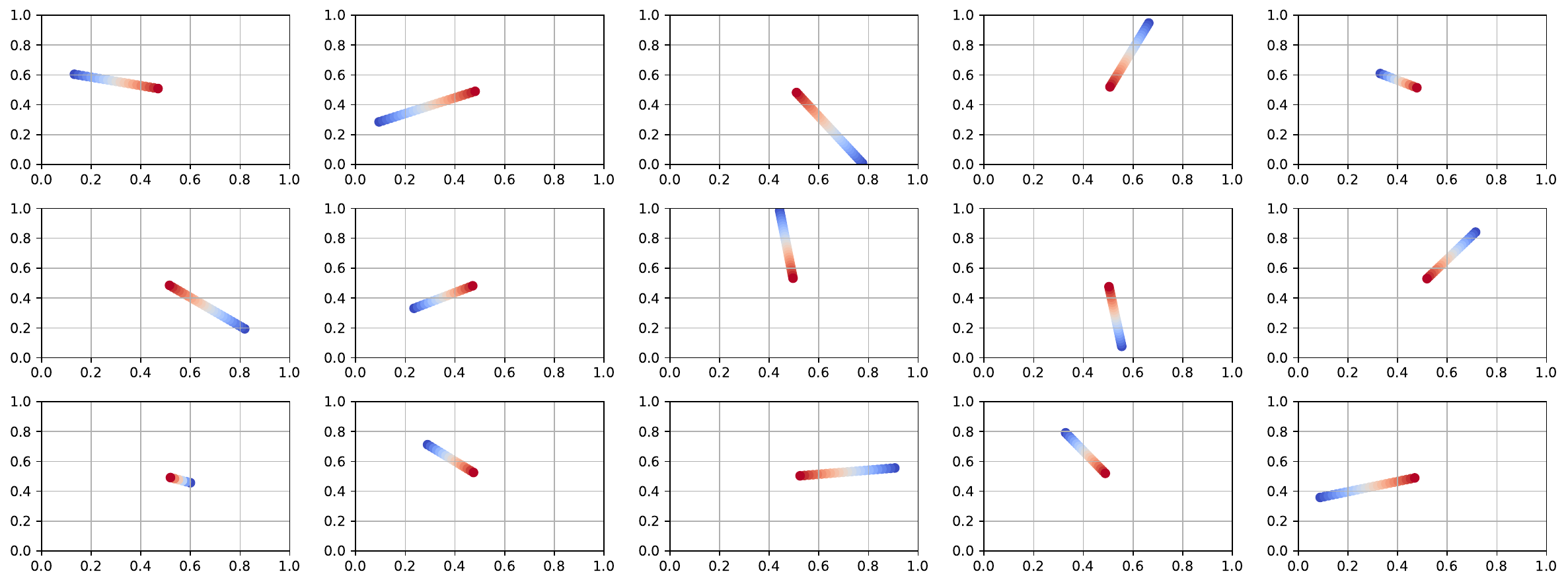}
    }
    \hfill
    \subfloat[Coordinate walk in $f_{\text{scale}}$]{
    \includegraphics[width=0.45\textwidth]{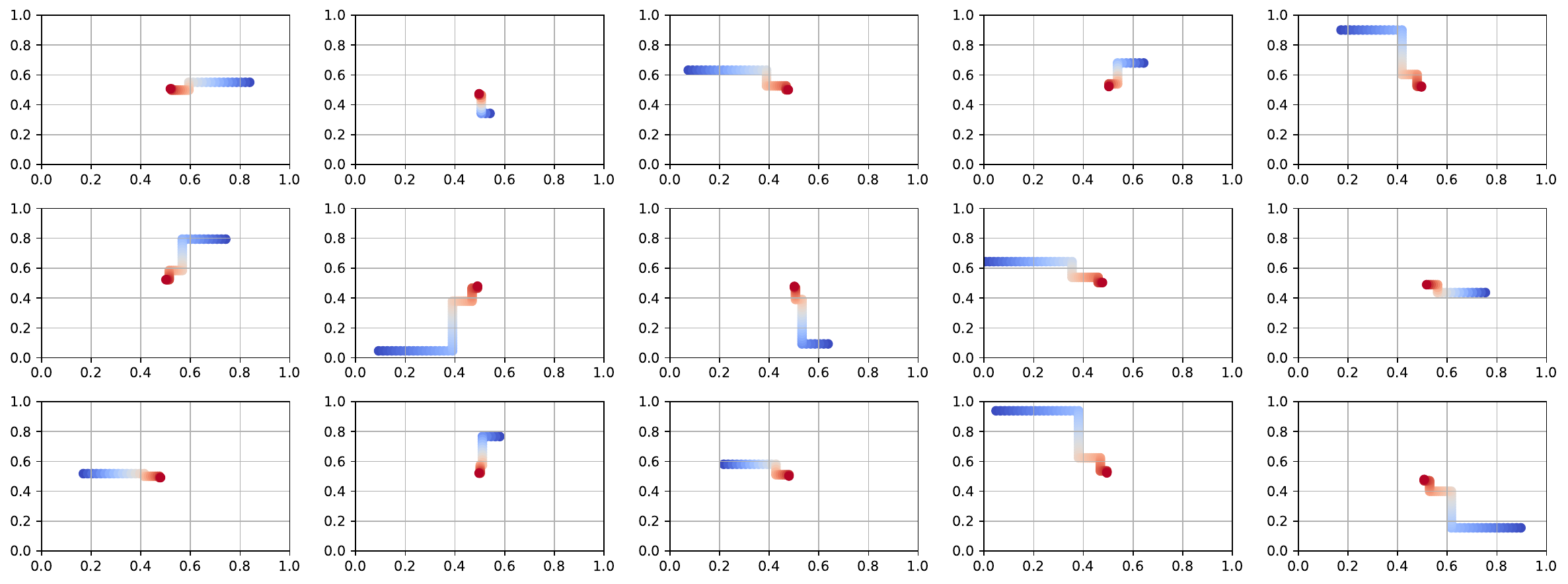}
    }

    \subfloat[Direct policy in $f_{\text{rot}}$]{
    \includegraphics[width=0.45\textwidth]{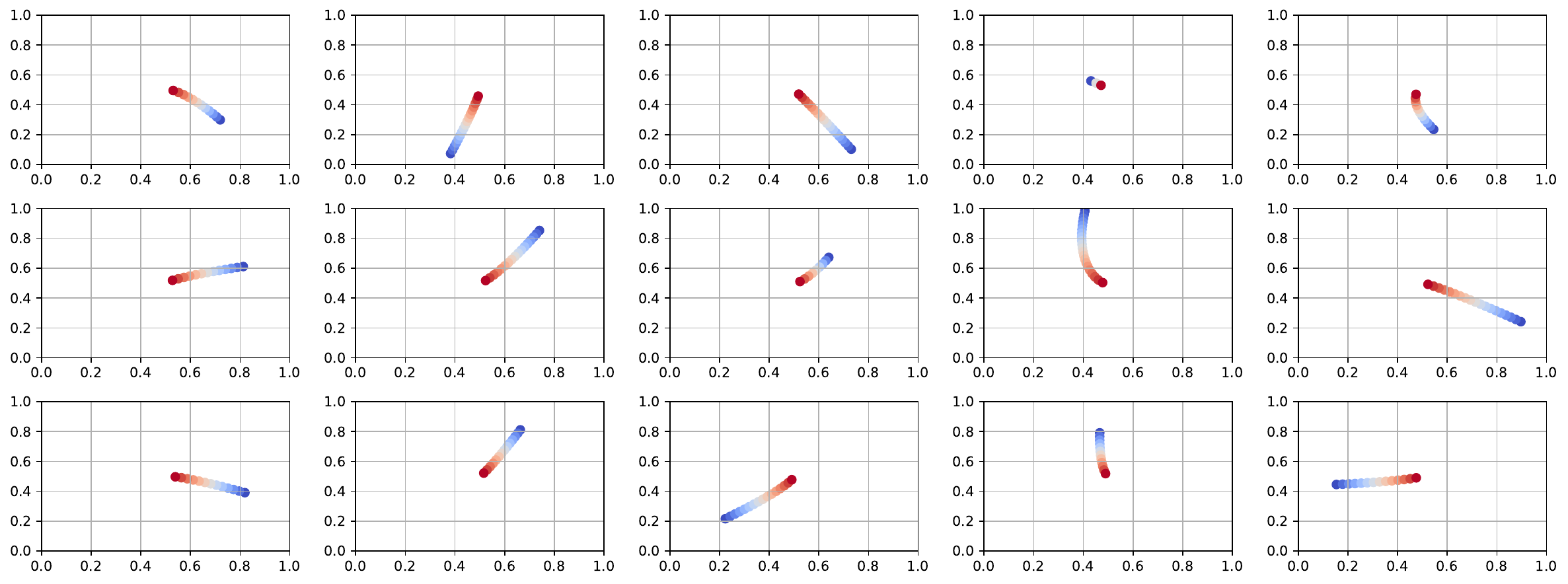}
    }
    \hfill
    \subfloat[Coordinate walk in $f_{\text{rot}}$]{
    \includegraphics[width=0.45\textwidth]{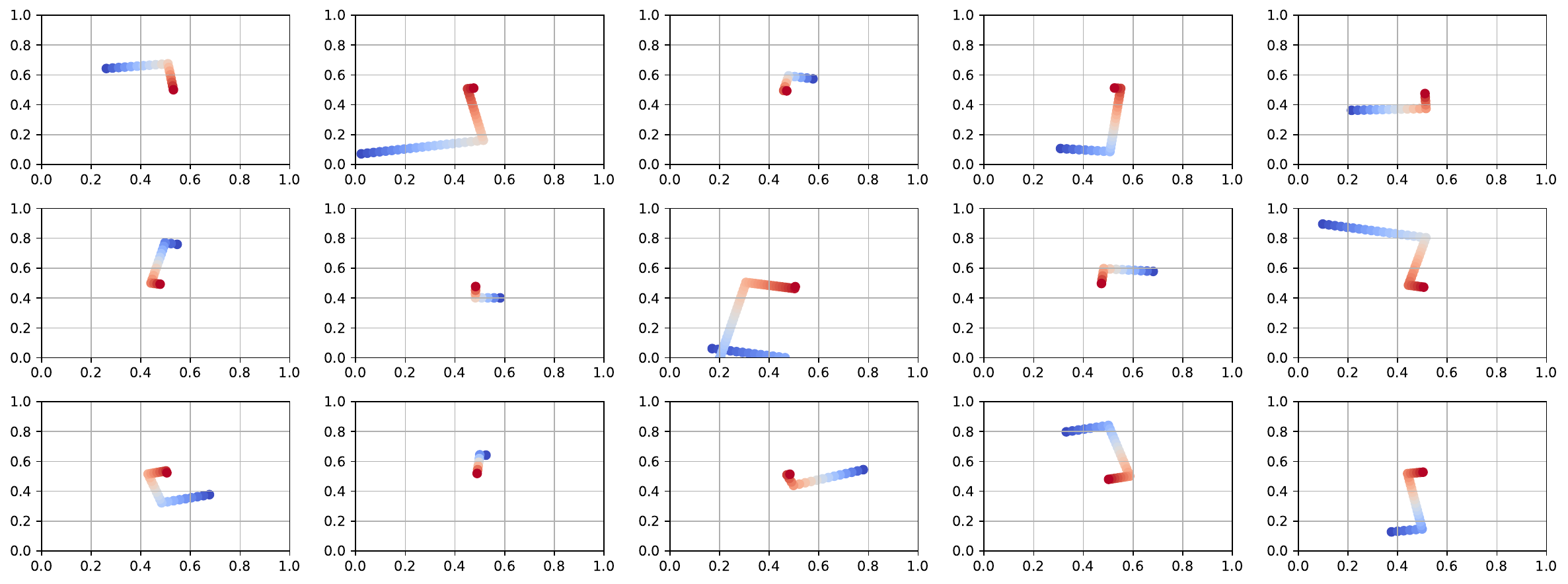}
    }

    \subfloat[Direct policy in $f_{\text{regrot}}$]{
    \includegraphics[width=0.45\textwidth]{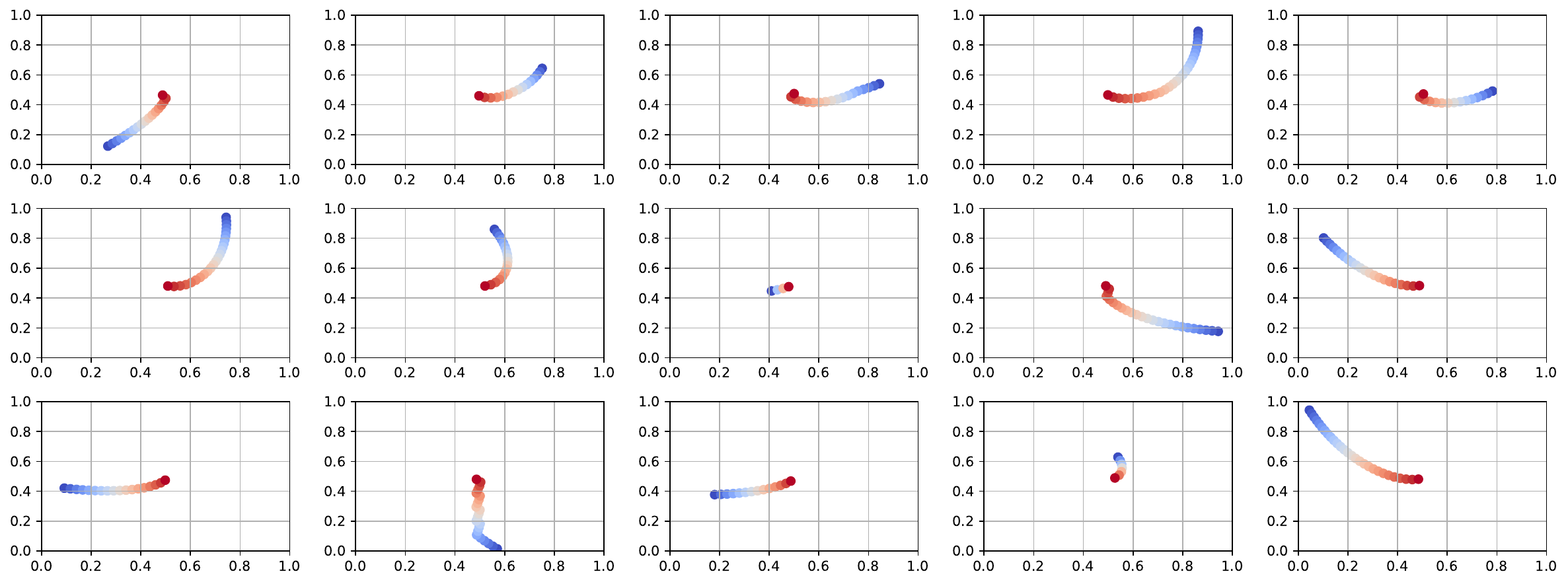}
    }
    \hfill
    \subfloat[Coordinate walk in $f_{\text{regrot}}$]{
    \includegraphics[width=0.45\textwidth]{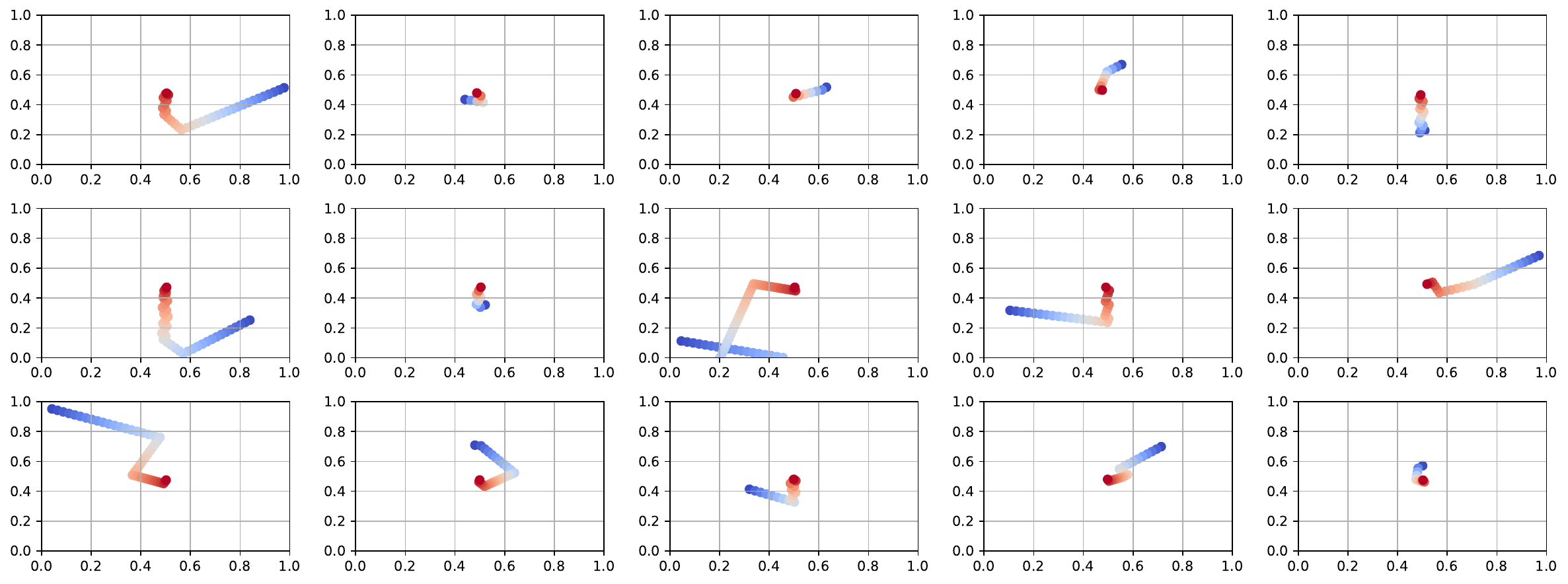}
    }

    \subfloat[Direct policy in $f_{\text{sin}}$]{
    \includegraphics[width=0.45\textwidth]{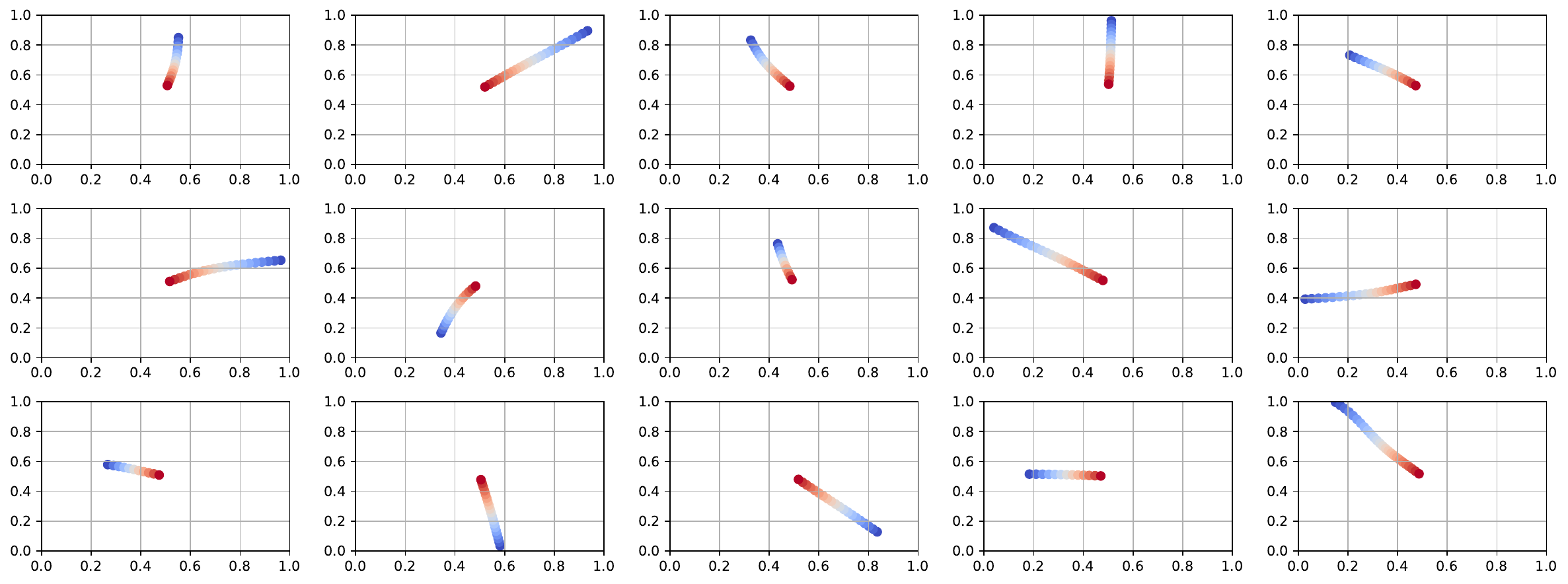}
    }
    \hfill
    \subfloat[Coordinate walk in $f_{\text{sin}}$]{
    \includegraphics[width=0.45\textwidth]{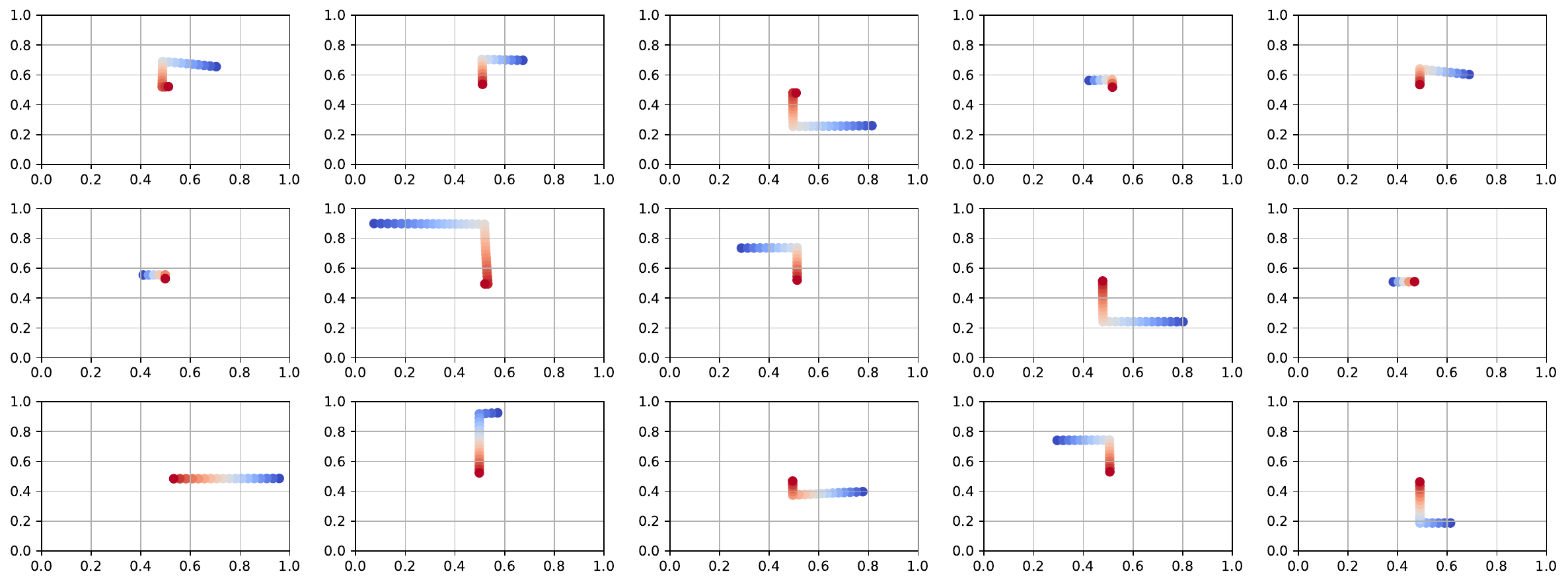}
    }

    \subfloat[Direct policy in $f_{\text{sqrt}}$]{
    \includegraphics[width=0.45\textwidth]{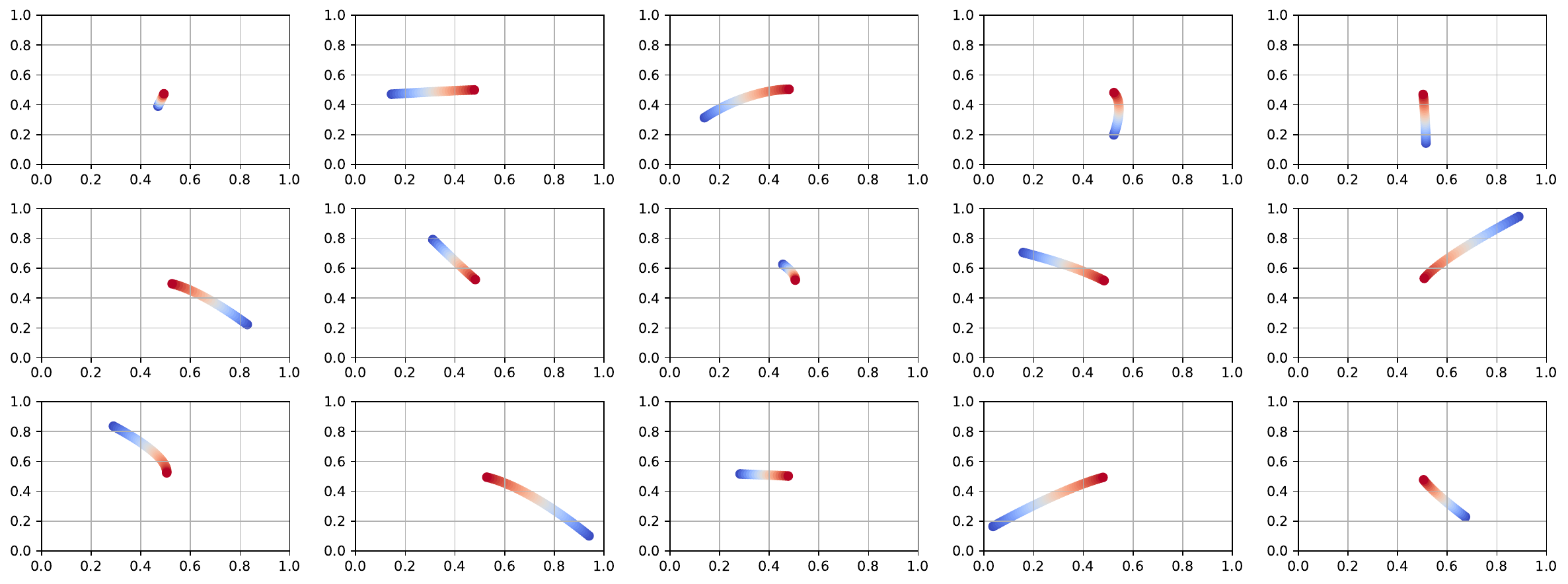}
    }
    \hfill
    \subfloat[Direct walk in $f_{\text{sqrt}}$]{
    \includegraphics[width=0.45\textwidth]{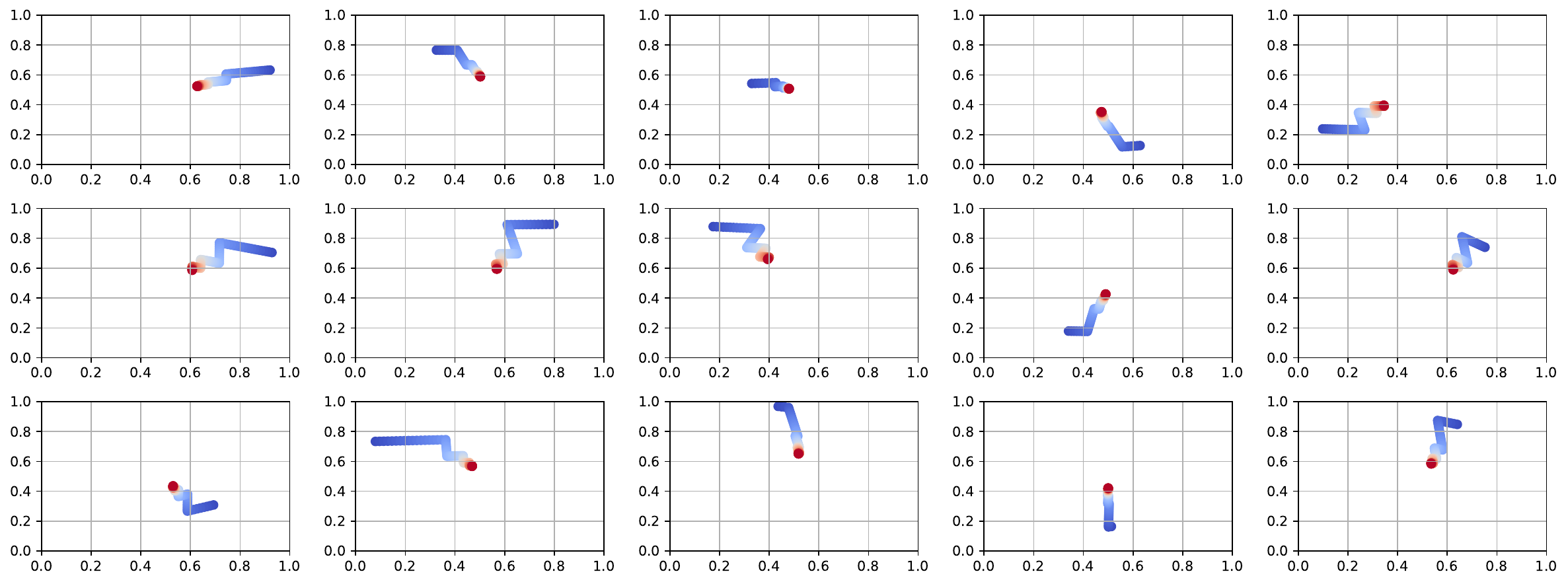}
    }
    \caption{Trajectories of direct policy and coordinate walk in different movement dynamics.}
    \label{fig:example_trajectories}
\end{figure}
Under mild distortions and additional assumptions on the distribution of $s_W$, the direct policy is the optimal policy

 the optimal behavior under mild distortion and additional assumptions on the
distribution of $s_W$. First, in case of full observability, the optimal policy is as follows:
\begin{prop}\label{prop:optimal_policy_full_observability}
Under full observability, i.e., $O(s, W)=(s,W)$, the optimal policy $\pi^*_\lambda$ is given by
$$\pi^*_\lambda(s, W) = \arg\min_{\|a\|\le \lambda} \|f(s, a, W) - s_W\|.$$
\end{prop}
\begin{proof}
    First, we define the state-action value
functions $Q^\pi(s, a, W)$ and $Q^\pi(s, a)$ similarly to the value functions $V^\pi(s, W)$ and
$V^\pi(s)$ from Section~\ref{sec:active_positioning}.
Clearly, the policy $\pi^*_\lambda$ is the policy yielding the maximal expected reward in each
step. This is due to the fact as it gets
closest to the terminal state $s_W$ and the reward depends only on the distance to $s_W$.
Thus
$$\max_{a} Q^\pi(s, a, W)\le\max_{a} Q^{\pi^*_\lambda}(s, a, W)$$
for any state $(s, W)$ and the same holds for the expected values over $W\sim\cW$, i.e.,
$\max_{a} Q^\pi(s, a)\le\max_{a} Q^{\pi^*_\lambda}(s, a)$.
\end{proof}

Clearly, the policy $\pi^*_\lambda$ from Proposition~\ref{prop:optimal_policy_full_observability} is not
applicable in practice as neither the context $W$ is observed nor the movement dynamics $f$
is explicitly known which is needed to solve the minimization problem in each step.
In case only $s$ is observed as in $\mathcal{O}_{\text{PO}}$, the best action a policy can take is the one
where the expected distance to the terminal state over all contexts $W$ is minimized, that is:
$$\pi^*_\lambda(s) = \arg\min_{\|a\|\le \lambda} \mathbb{E}_{W\sim\cW}[\|f(s, a, W) -s_W\|].$$
Still, without further assumptions on $f$, $s_W$, and $\cW$, computing $\pi^*_\lambda(s)$ is intractable.
However, assuming the expected value of $s_W$ exists and is available and that
the placement error does not depend on the state, i.e., $f(s, a, w)=s+g(a, W)$, the optimal is
explicitly given as follows:
\begin{prop}
    Let $f(s, a, W)=s+g(a, W)$ with $\mathbb{E}_{W\sim\cW}[g(a, W)]=a$ and assume
    that $\mathbb{E}_{W\sim\cW}[s_W]=s^*$.  Then the optimal policy is
    $\pi^*_\lambda(s) = \mathrm{clip}_\lambda (s^*-s)$.
\end{prop}
\begin{proof}
    We have
    \begin{align*}
        \pi^*_\lambda(s) &= \arg\min_{\|a\|\le \lambda} \mathbb{E}_{W\sim\cW}[\|f(s, a, W) -s_W\|]\\
                  &= \arg\min_{\|a\|\le \lambda} \mathbb{E}_{W\sim\cW}[\|s+g(a, W) -s_W\|]\\
                  &= \arg\min_{\|a\|\le \lambda} [\|s+a -s^*\|]\\
                  &= \mathrm{clip}_\lambda (s^*-s)
    \end{align*}
\end{proof}

\section{Additional experiments in Fetch-environment}\label{app:fetch}

In extension to the reach experiments in Section~\ref{sec:experiments} where the positional
differences are directly observed, we provide in this section a proof of principle that shortcut
augmentations can also benefit offline RL methods in more involved robotic environments. 
To this end, we consider two scenarios based on the
Fetch environment~\cite{fetch}. In the first scenario, we study a reaching task in which the robotic arm must reach a target
position in 3D space. The observation is an image of the scene. We collect $100$ trajectories using the coordinate walk policy
described in Section~\ref{s:logging_policy}.

\begin{figure}[H]
    \centering

    \hfill
    \subfloat[Image-based reaching in $d=3$ with $100$ trajectories]{
    \includegraphics[width=0.4\textwidth]{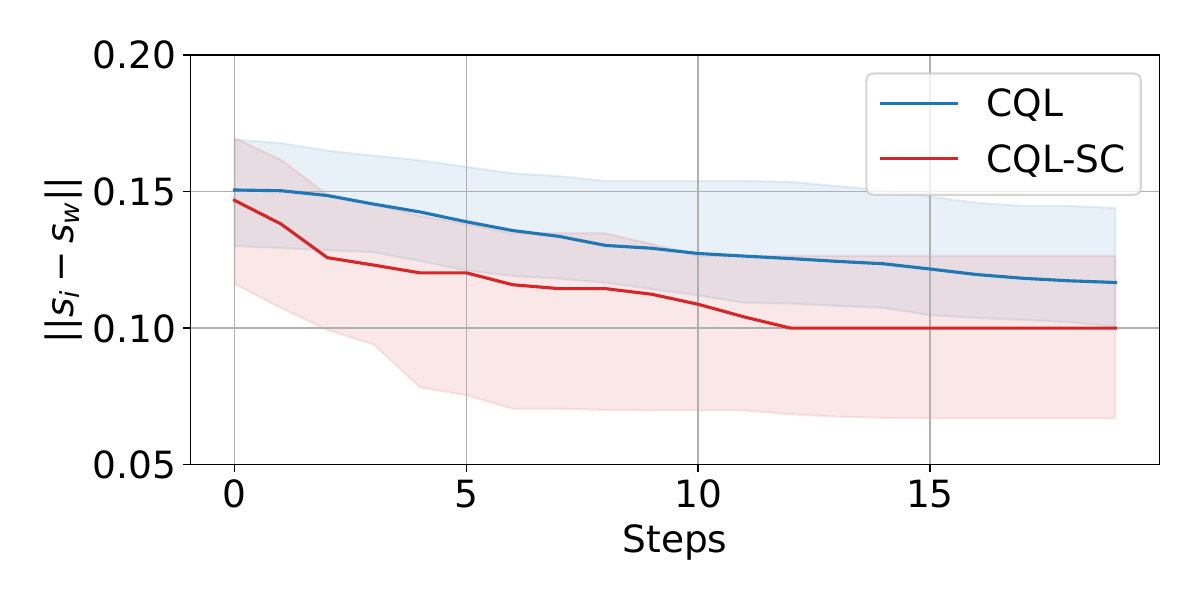}
    }
    \hfill
    \subfloat[Position based pick-and-place in $d=3$ with $1000$ trajectories]{
    \includegraphics[width=0.4\textwidth]{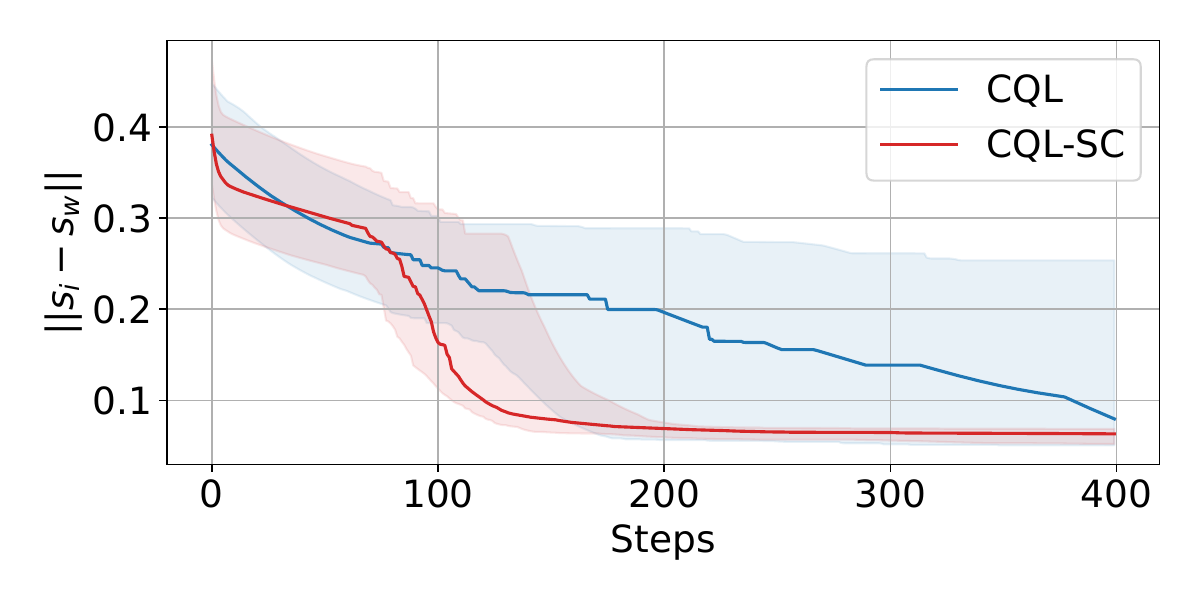}
    }
    \hfill
    \caption{Experiments in the Fetch environment.}\label{fig:fetch_experiments}.
\end{figure}

In the second scenario, we consider a variant of the pick-and-place task where the robotic arm must
move an object from a random initial position to a random target position. We focus solely on the
positioning, i.e., the object does not need to be grasped, only touched, assuming perfect gripper control.
The policy used here performs two consecutive coordinate walks: one to reach the object and one to
reach the target position. The observations are given by the distances from the gripper to the
object and from the gripper to the target where the first distance is zeroed once solves the
touching task. In this setting, we collect $1000$ trajectories.
On the collected datasets, we train CQL both with and without shortcuts, and the results are
reported in Figure~\ref{fig:fetch_experiments}.

\section{Details for Experimental Results}\label{app:additional_experiments}

\subsection{Hyperparameters of learning algorithms}\label{sec:effect_of_hyperparameters}

\begin{minipage}[t]{0.30\textwidth}
    \centering
    \begin{tabular}{c l}
    \hline
    \textbf{Parameter} & \textbf{Value} \\
    \hline
    actor learning rate  & $ 10^{-3}$ \\
    critic learning rate & $10^{-3}$ \\
    conservative weight   & $5.0$ \\
    $\alpha$-threshold       & $10.0$ \\
    batch size            & $500$ \\
    $\gamma$                  & $0.99$ \\
    $\tau$                    & $0.005$ \\
    \hline
    \end{tabular}
    \captionof{table}{Parameter for CQL trained on collected datasets.}
\end{minipage}
\hfill
\begin{minipage}[t]{0.30\textwidth}

    \centering
    \begin{tabular}{c l}
    \hline
    \textbf{Parameter} & \textbf{Value} \\
    \hline
    actor learning rate  & $ 10^{-3}$ \\
    critic learning rate & $10^{-3}$ \\
    conservative weight   & $5.0$ \\
    $\alpha$-threshold       & $10.0$ \\
    batch size            & $500$ \\
    $\gamma$                  & $0.99$ \\
    $\tau$                    & $0.005$ \\
    \hline
    \end{tabular}
    \captionof{table}{Parameter for CQL trained as \ourmethod{} augmentor.}
\end{minipage}
\hfill
\begin{minipage}[t]{0.30\textwidth}

    \begin{tabular}{c l}
    \hline
    \textbf{Parameter} & \textbf{Value} \\
    \hline
    actor learning rate  & $ 10^{-3}$ \\
    critic learning rate & $10^{-3}$ \\
    batch size            & $256$ \\
    n updates per step   & $5$ \\
    n critics & $2$ \\
    $\gamma$                  & $0.99$ \\
    $\tau$                    & $0.005$ \\
    \hline
    \end{tabular}
    \captionof{table}{Parameter for SAC.}

\end{minipage}

\subsection{Hyperparameter study of \ourmethod{}}

In this section, we study effects of the different hyperparameters of the shortcut computation
(Algorithm~\ref{alg:shortcut_computation}) and \ourmethod{} (Algorithm~\ref{alg:data_augmentation}).
First, we study the effect of the number of augmentations per trajectory $n$ and the probability of
applying an augmentation $p$. The results are shown in Figure~\ref{fig:hyperparameter_n}.
One can see that as few as $20$ augmentations per trajectory are sufficient to achieve a substantial
improvement in performance, provided that the augmentation probability is not too low. Notably,
higher probabilities correspond to augmentations being applied earlier in the trajectory. This
suggests that augmentations at the beginning of a trajectory are more beneficial than those applied
later.

\begin{figure}[H]
    \centering
    \includegraphics[width=0.99\textwidth]{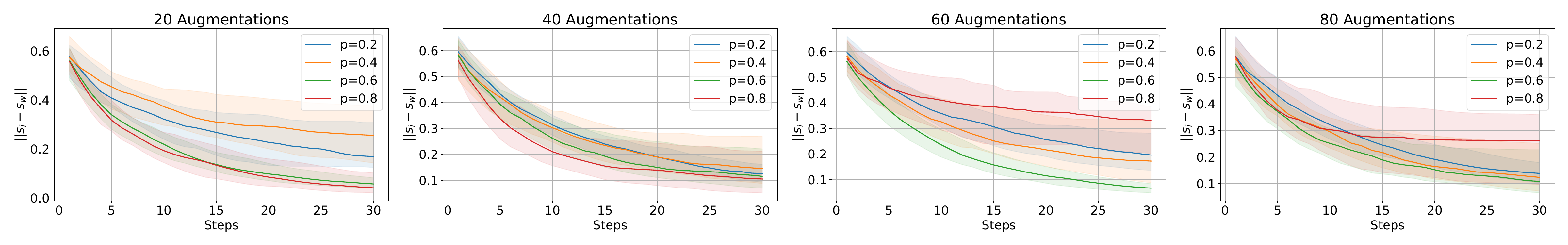}
    \caption{Experiments in $f_{\text{blend}}$ with step size $0.025$ and different probabilities
    $p$ of applying augmentations and different maximal number of augmentations per
trajectory}\label{fig:hyperparameter_n}
\end{figure}

Next, we analyse the effect of the sampling scheme of shortcuts along a trajectory. Here, we denote
the sampling mechanism described in Algorithm~\ref{alg:shortcut_computation} as \emph{weighted}.
Another way to sample shortcuts from the set $S$ computed in
Algorithm~\ref{alg:shortcut_computation} is to use a distribution that is proportional to the
inverse distance to the optimum, i.e. $p(i)\sim \frac{1}{\|s_i-s_W\|}$ or to sample uniformly from
$S$. Instead of sampling, one can also just use the shortcut residing within the action space that
leads to the point of highest reward within the trajectory called \emph{best}. The results are shown
in Figure~\ref{fig:hyperparameters_p} for $n=20$ augmentations per trajectory and $p=0.4$ showing
that in the environments we consider, the sampling strategy does not have a significant effect on
the performance.

\begin{figure}[H]
    \centering
    \includegraphics[width=0.99\textwidth]{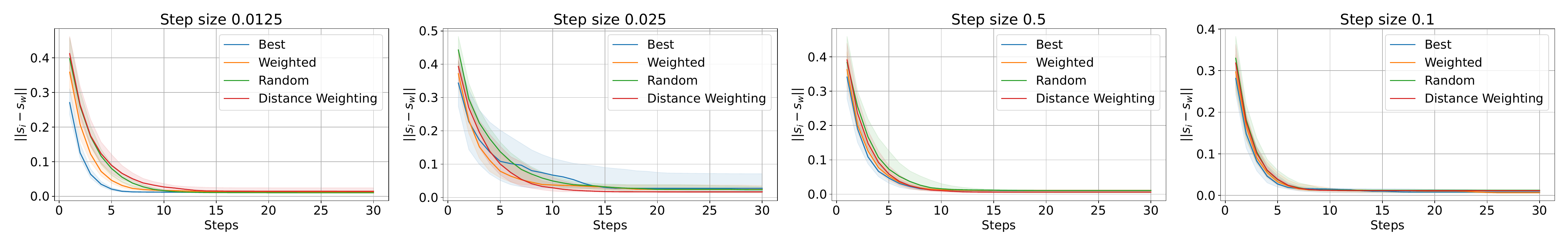}
    \hfill
    \caption{Experiments in $f_{\text{blend}}$ with different step size and different sampling
    strategies.}\label{fig:hyperparameters_p}
\end{figure}

\subsection{Comparison of \ourmethod{} and SAC}\label{sec:sac_vs_lift}

Table~\ref{tab:sac_vs_lift} summarizes settings in which \ourmethod{}-SC achieves a smaller distance
to the optimum than the SAC baseline after 30 interaction steps in 
environment $\mathcal{O}_{\text{PO}}$ with dimensionality $d=5$, across different step sizes of the
logging policy and movement distortions. Figures \ref{fig:movement}--\ref{fig:sqrt} provide a complete comparison of all methods
over the first 30  steps, showing the median distance to the target across multiple
runs.

\begin{table}[H]
  \centering
  \begin{tabular}{l p{0.5cm} p{0.5cm} p{0.5cm} p{0.5cm}}
    $\cw{l}$ & $.0125$ & $.025$ & $.05$ & $.1$ \\
    \toprule
    $f_{\text{blend}}$ & $\bullet$ & $\bullet$ & $\bullet$ & $\bullet$  \\
    $f_{\text{scale}}$ & $\bullet$ & $\bullet$ & $\bullet$ & $\bullet$  \\
    $f_{\text{rot}}$ & $\bullet$ & $\bullet$ & $\bullet$ & $\bullet$  \\
    $f_{\text{regrot}}$ & $\bullet$ &  & &  \\
    $f_{\text{sin}}$ & $\bullet$ & $\bullet$ & $\bullet$ & $\bullet$  \\
    $f_{\text{sqrt}}$ &  & $\bullet$ & $\bullet$ & $\bullet$ 
  \end{tabular}
  \caption{Cases where \ourmethod{}-SC outperforms SAC baseline in
  $\mathcal{O}_{\text{PO}}$, $d=5$.}
  \label{tab:sac_vs_lift}
\end{table}

\subsection{Ablation on structure of logging policy\label{sec:ablation_logging_policy}}

In this section, we analyse the effect of absence of structure in the logging policy on the
performance of the shortcut augmentation by injecting noise into the $\cw{l}$.
Specifically, we used $\mathcal{O}_{\mathrm{PO}}$ under three different dynamics. At each step of
the coordinate-walk logging policy, we added Gaussian noise to the action and considered a range of
noise levels, from $\lambda=0$ (the original coordinate walk) up to $\lambda=2$, where the behavior
is close to a random walk and little of the original coordinate structure remains visible (see
Figure~\ref{fig:randomized_trajectories}). We then train and evaluate three CQL models with and
without shortcut augmentation respectively on datasets generated by these noisy-variant of $\cw{l}$.
The results are in shown in Figure~\ref{fig:ablation_logging_policy}: Across all tested scenarios,
shortcut augmentation consistently yields substantially better policies, suggesting that the method
is not limited to highly structured logging policies. 

\begin{figure}[H]
    \centering
    \includegraphics[width=0.99\textwidth]{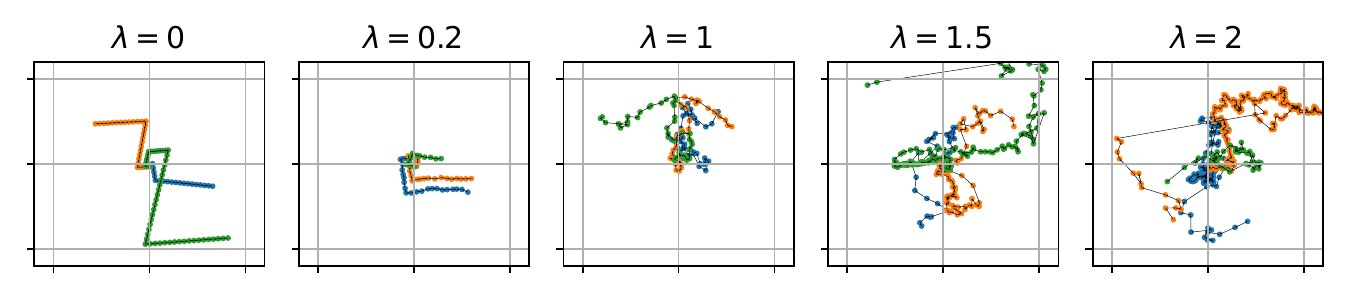}
    \caption{Comparison of different logging policies in $f_{\text{blend}}$ with $d=5$ and step size
    $0.05$.}\label{fig:randomized_trajectories}
\end{figure}

\begin{figure}[H]
    \centering
    \includegraphics[width=0.99\textwidth]{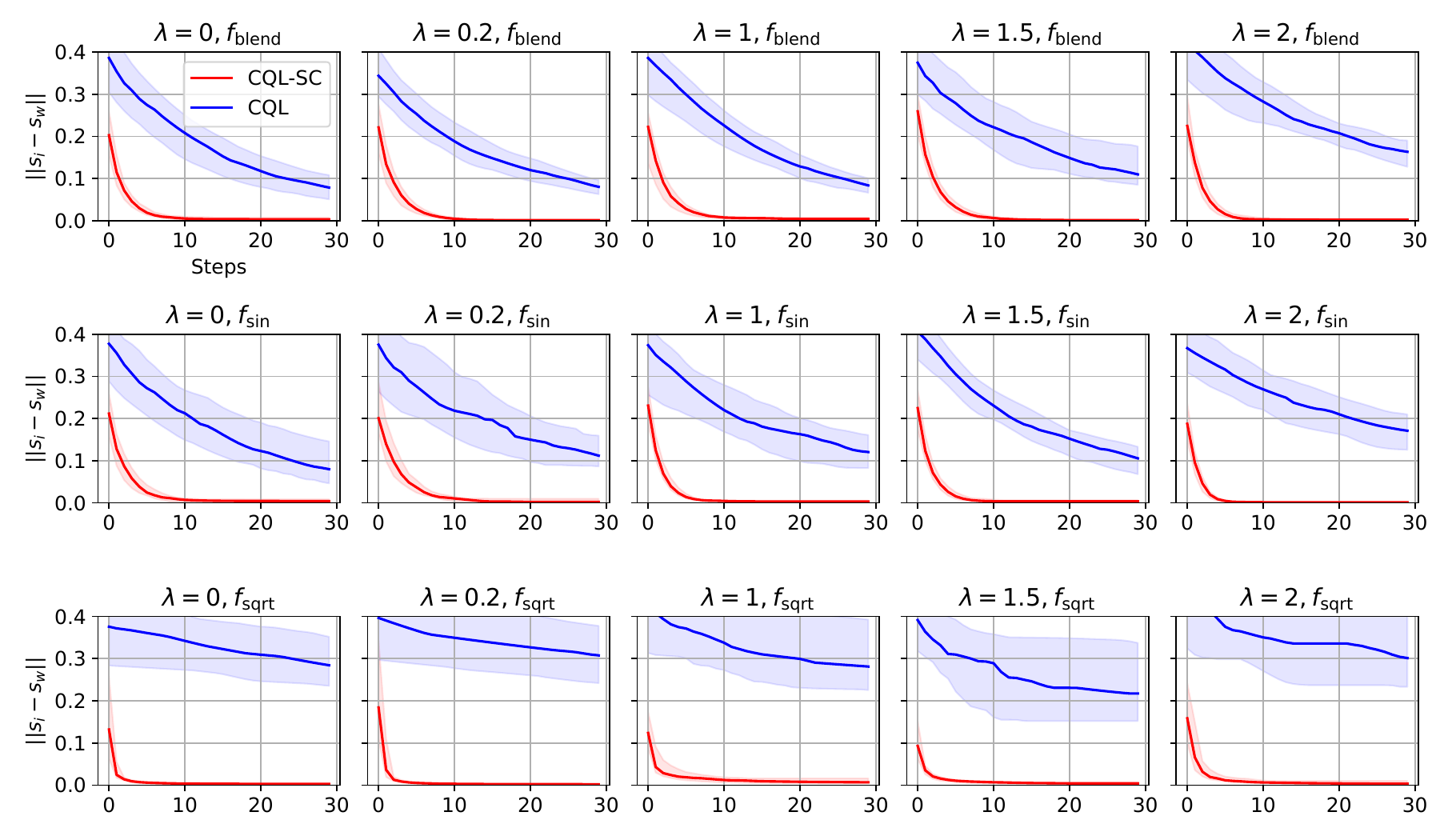}
    \caption{Comparison of noisy $\cw{l}$ with different noise levels $\lambda$ for different
        movement distortions.}\label{fig:ablation_logging_policy}
\end{figure}

\subsection{Analysis of the Influence of $C$}\label{sec:C_analyse}
In this section, we study the influence of the hyperparameter $C$ during shortcut computation
(Algorithm~\ref{alg:shortcut_computation}). Higher values of $C$ lead to more restrictive shortcut selection.

 \begin{figure}[ht]
    \centering
    \begin{tikzpicture}
        \node[] at (0,0) {\includegraphics[width=\columnwidth]{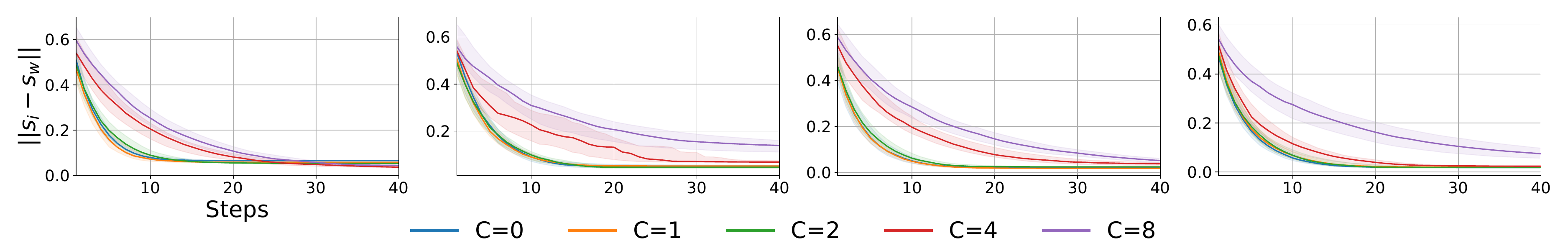}};
        \node[scale=0.7, align=center] at (-6, 1.4) {$\mathcal{O}_{\text{PO}}, f_{\text{regrot}}$ $d=5, l=0.1$};
        \node[scale=0.7, align=center] at (-2, 1.4) {$\mathcal{O}_{\text{PO}}, f_{\text{regrot}}$ $d=5, l=0.05$};
        \node[scale=0.7, align=center] at (2.2, 1.4) {$\mathcal{O}_{\text{PO}}, f_{\text{blend}}$ $d=5, l=0.1$};
        \node[scale=0.7, align=center] at (6.3, 1.4) {$\mathcal{O}_{\text{PO}}, f_{\text{blend}}$ $d=5, l=0.05$};
    \end{tikzpicture}
    \caption{Comparisons of our methods for selected scenarios.}
\end{figure}
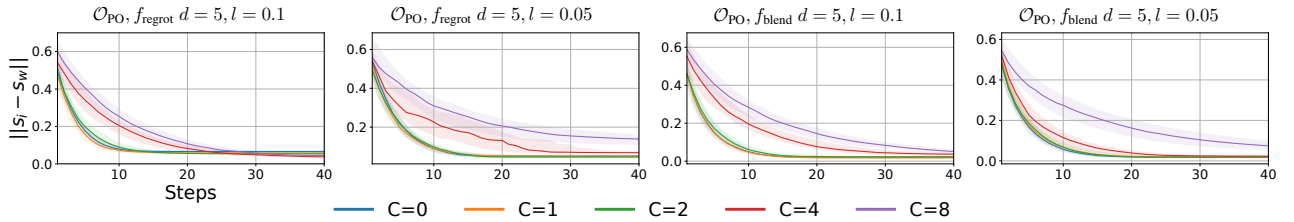

\begin{figure}[H]
    \centering

    \subfloat[$\mathcal{O}_{\text{PO}}, f_{\text{blend}}$ $d=5$]{
    \includegraphics[width=0.48\textwidth]{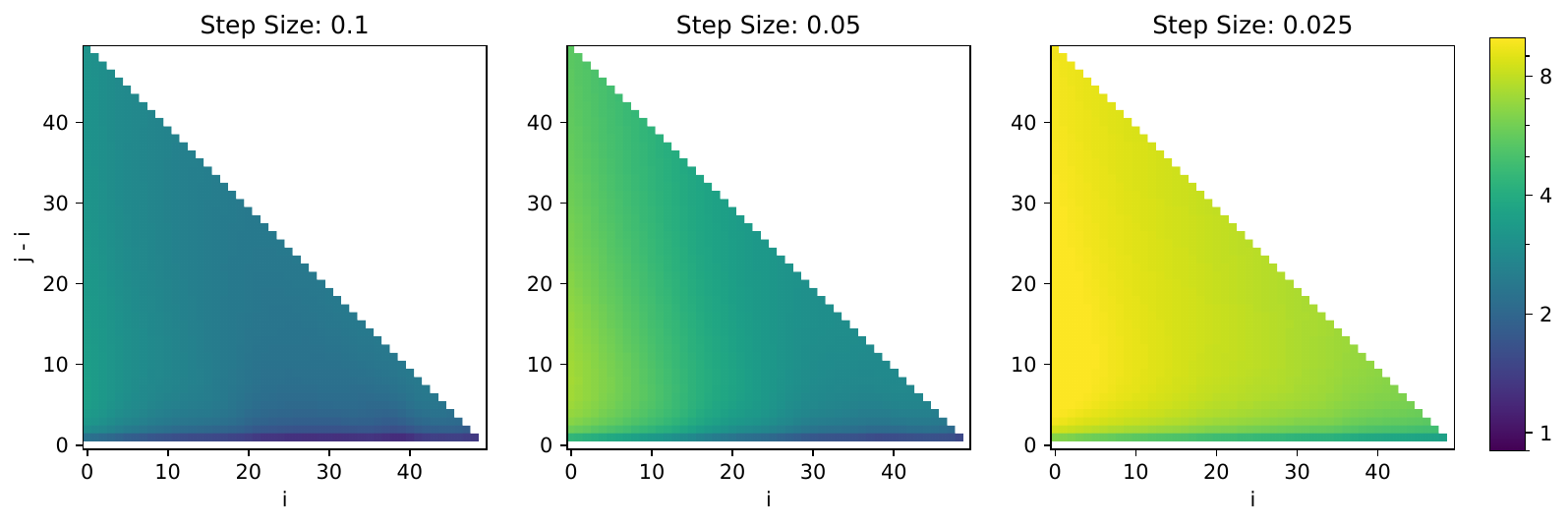}
    }
    \hfill
    \subfloat[$\mathcal{O}_{\text{PO}}, f_{\text{regrot}}$ $d=5$]{
    \includegraphics[width=0.48\textwidth]{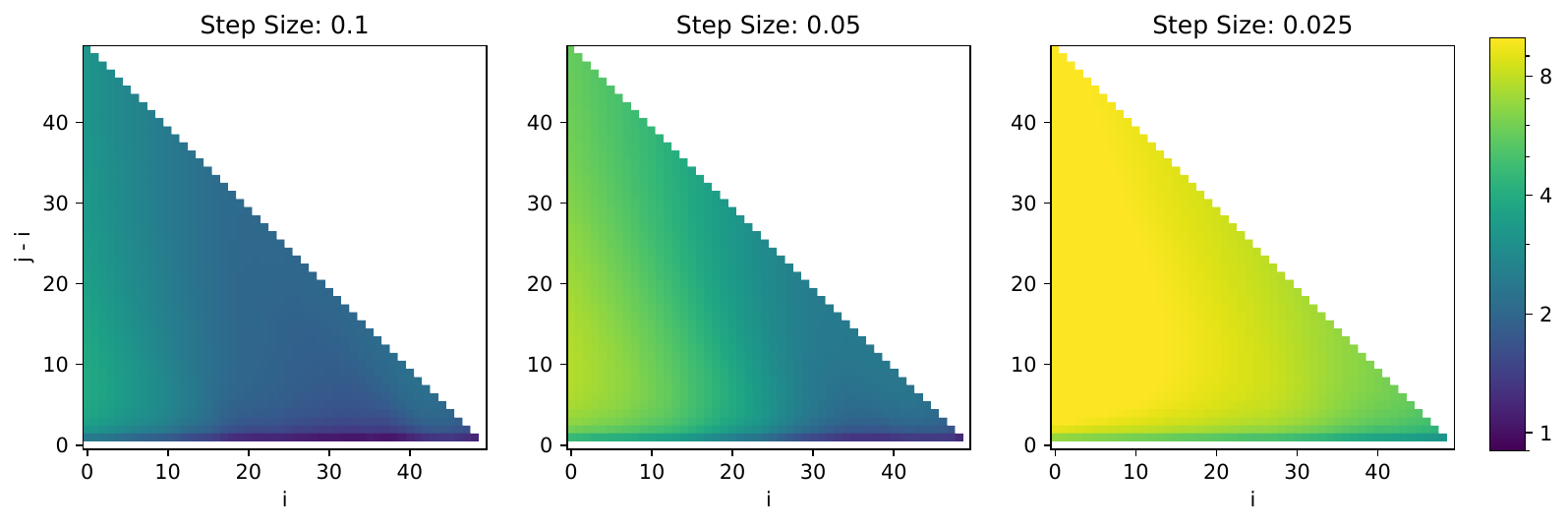}
    }
    \caption{
    Dependence which values of $C$ give valid shortcut from $i$ (x-axis) to $j-i$ (y-axis), averaged
    over $500$ episodes of $\mathcal{O}_{\text{PO}}$.
}
\end{figure}

\subsection{Additional visualization}

\begin{figure}[H]
    \centering

    \subfloat[100 Episodes $\mathcal{O}_{\text{PO}}$ with $d=2$]{
    \includegraphics[width=0.99\textwidth]{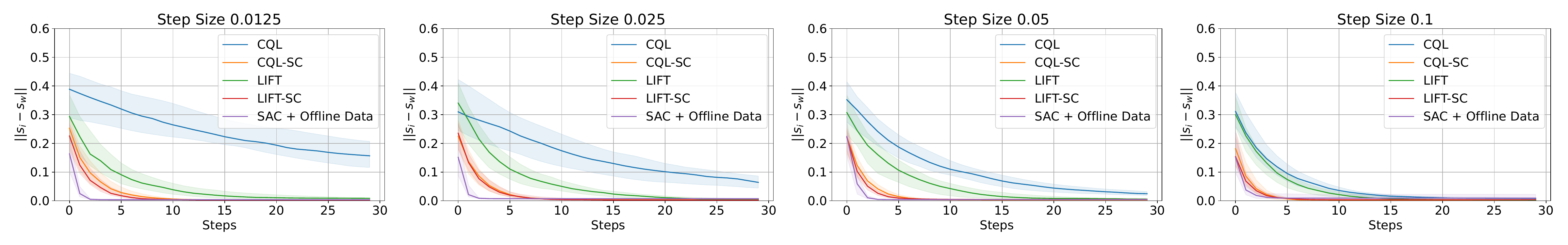}
    }
    \hfill
    \subfloat[500 Episodes $\mathcal{O}_{\text{PO}}$ with $d=5$]{
    \includegraphics[width=0.99\textwidth]{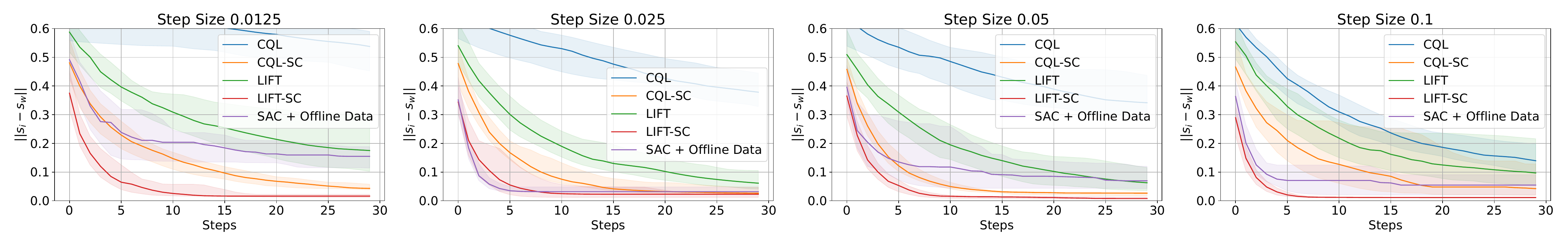}
    }
    \caption{Experiments in $f_{\text{blend}}$.}
    \label{fig:movement}
\end{figure}

\begin{figure}[H]
    \centering

    \subfloat[100 Episodes $\mathcal{O}_{\text{PO}}$ with $d=2$]{
    \includegraphics[width=0.99\textwidth]{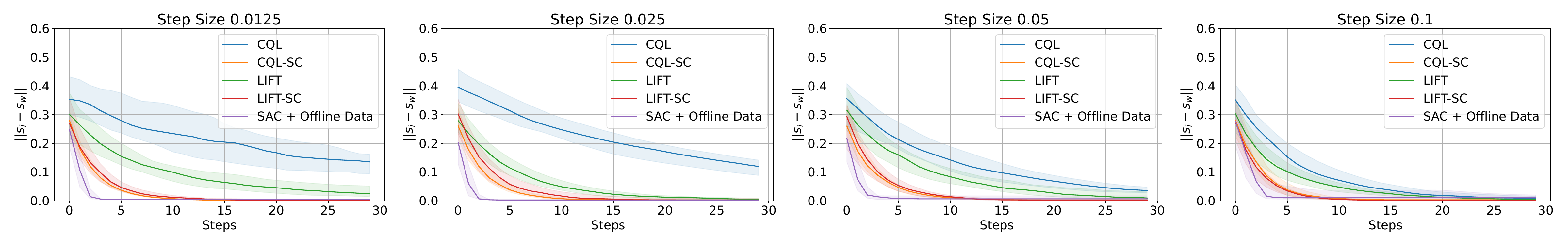}
    }
    \hfill
    \subfloat[500 Episodes $\mathcal{O}_{\text{PO}}$ with $d=5$]{
    \includegraphics[width=0.99\textwidth]{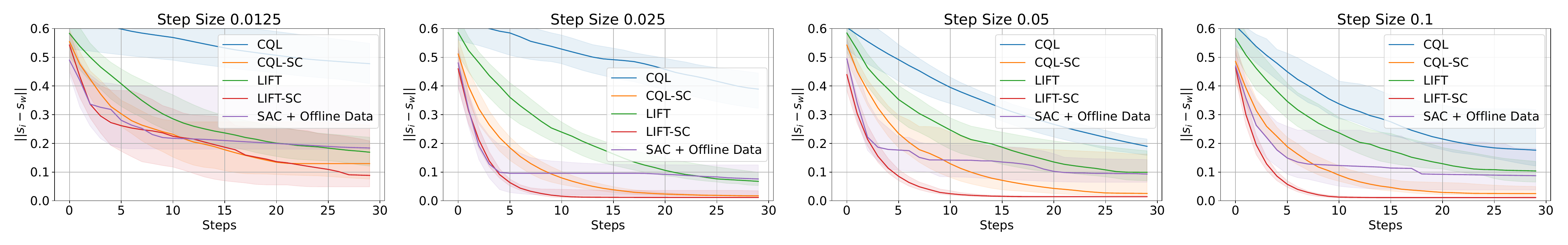}
    }
    \caption{Experiments in $f_{\text{scale}}$.}
\end{figure}

\begin{figure}[H]
    \centering

    \subfloat[100 Episodes $\mathcal{O}_{\text{PO}}$ with $d=2$]{
    \includegraphics[width=0.99\textwidth]{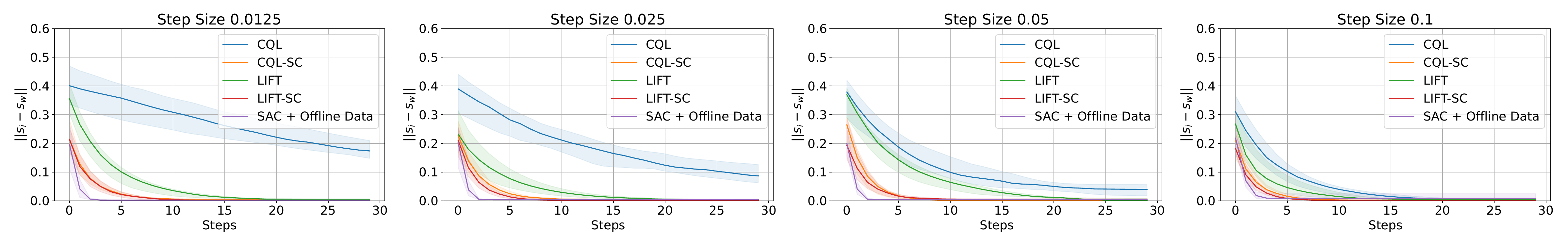}
    }
    \hfill
    \subfloat[500 Episodes $\mathcal{O}_{\text{PO}}$ with $d=5$]{
    \includegraphics[width=0.99\textwidth]{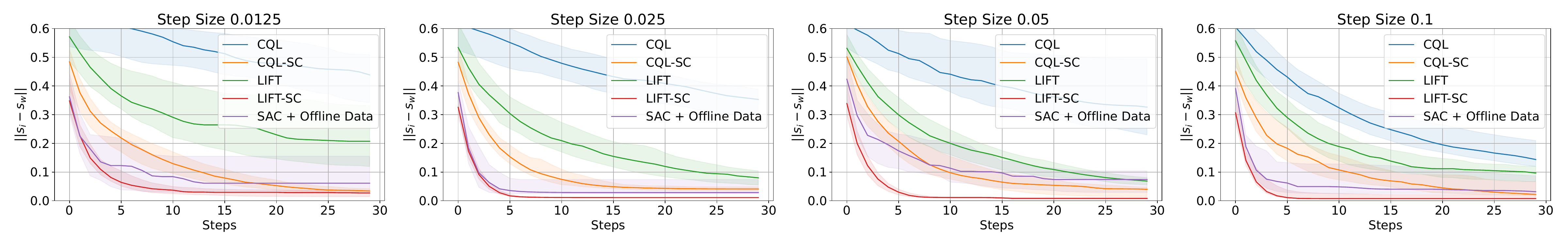}
    }
    \caption{Experiments in $f_{\text{rot}}$.}
\end{figure}

\begin{figure}[H]
    \centering

    \subfloat[100 Episodes $\mathcal{O}_{\text{PO}}$ with $d=2$]{
    \includegraphics[width=0.99\textwidth]{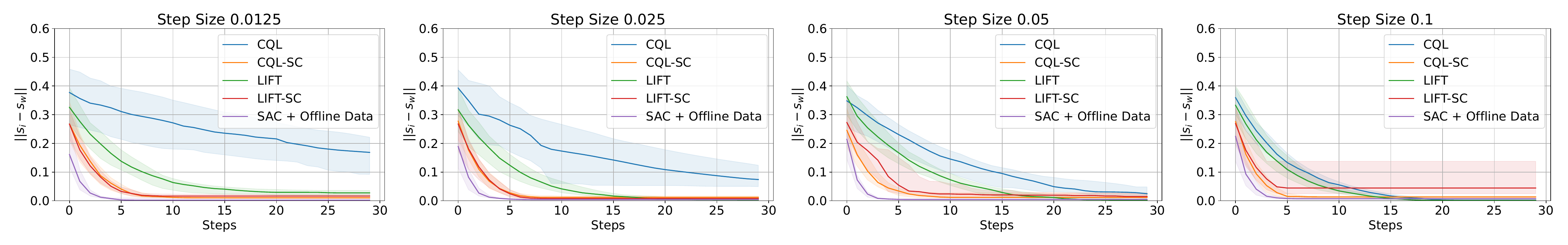}
    }
    \hfill
    \subfloat[500 Episodes $\mathcal{O}_{\text{PO}}$ with $d=5$]{
    \includegraphics[width=0.99\textwidth]{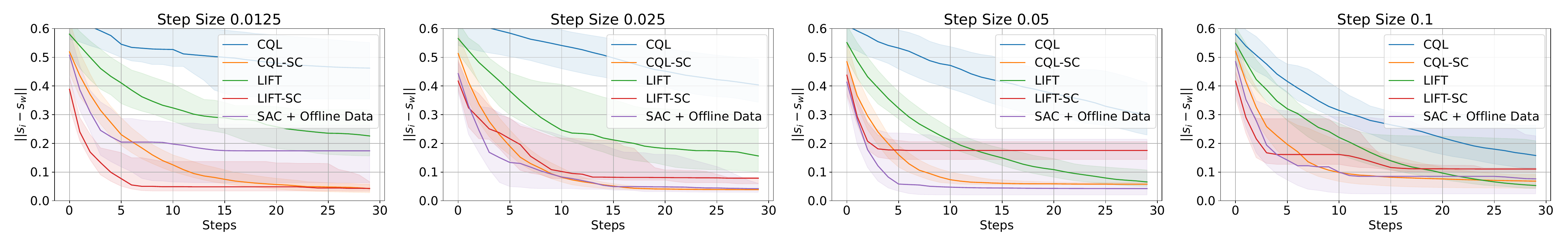}
    }
    \caption{Experiments in $f_{\text{regrot}}$.}
\end{figure}

\begin{figure}[H]
    \centering

    \subfloat[100 Episodes $\mathcal{O}_{\text{PO}}$ with $d=2$]{
    \includegraphics[width=0.99\textwidth]{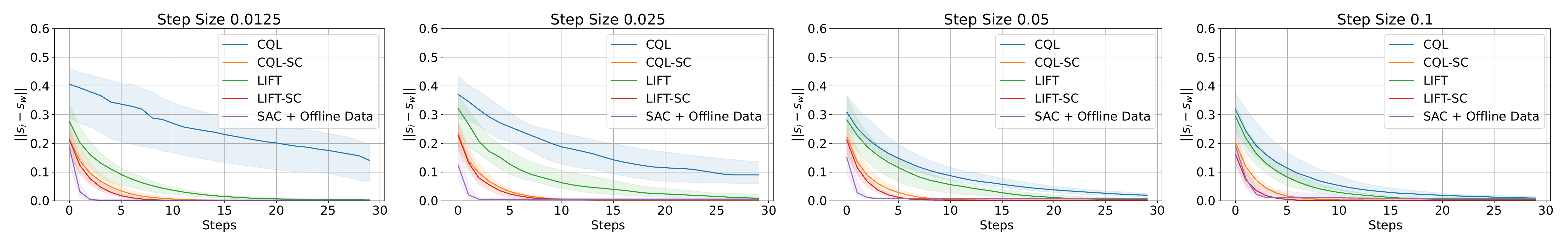}
    }
    \hfill
    \subfloat[500 Episodes $\mathcal{O}_{\text{PO}}$ with $d=5$]{
    \includegraphics[width=0.99\textwidth]{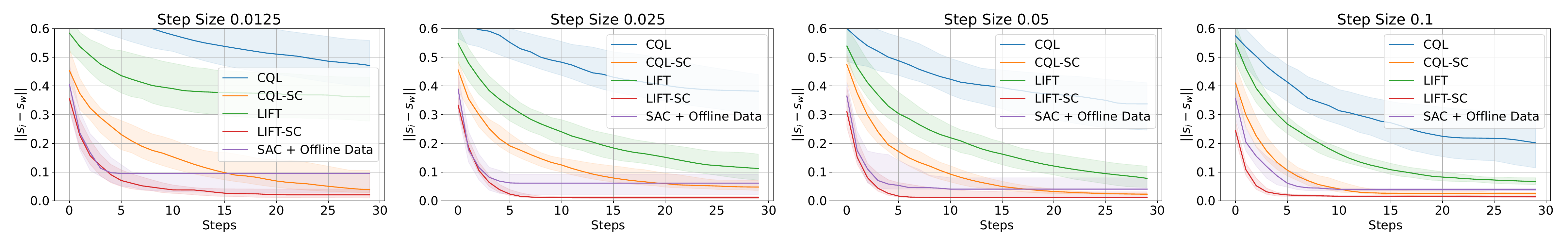}
    }
    \caption{Experiments in $f_{\text{sin}}$.}
\end{figure}

\begin{figure}[H]
    \centering

    \subfloat[100 Episodes $\mathcal{O}_{\text{PO}}$ with $d=2$]{
    \includegraphics[width=0.99\textwidth]{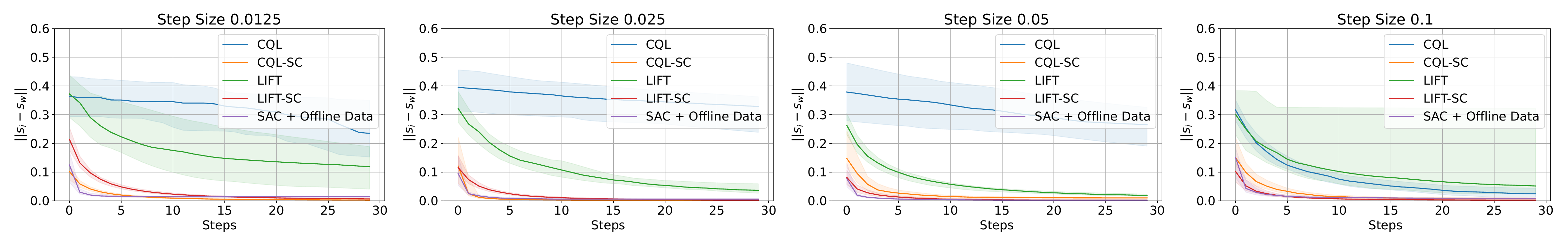}
    }
    \hfill
    \subfloat[500 Episodes $\mathcal{O}_{\text{PO}}$ with $d=5$]{
    \includegraphics[width=0.99\textwidth]{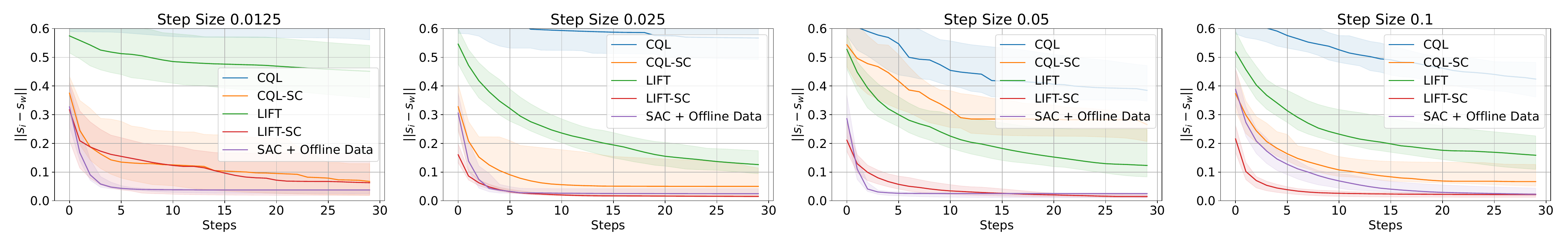}
    }
    \caption{Experiments in $f_{\text{sqrt}}$.}
    \label{fig:sqrt}
\end{figure}

\begin{figure}[H]
    \centering

    \subfloat[]{
    \includegraphics[width=0.27\textwidth]{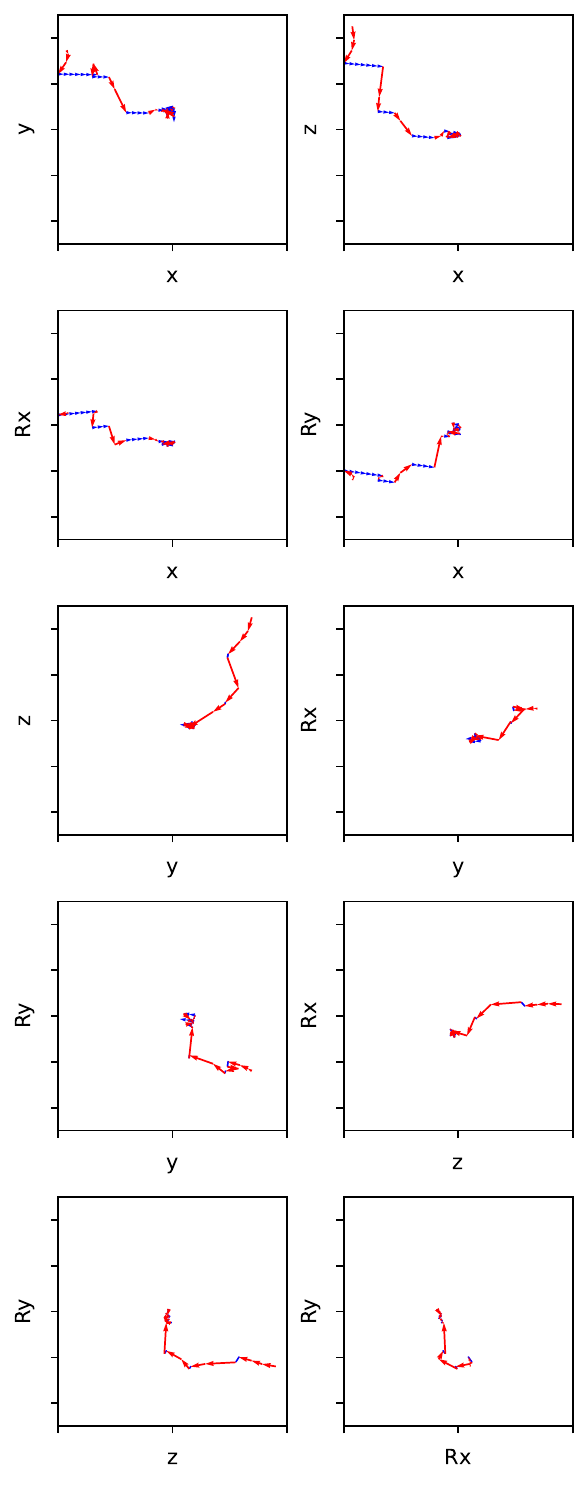}
    }
    \hfill
    \subfloat[]{
    \includegraphics[width=0.27\textwidth]{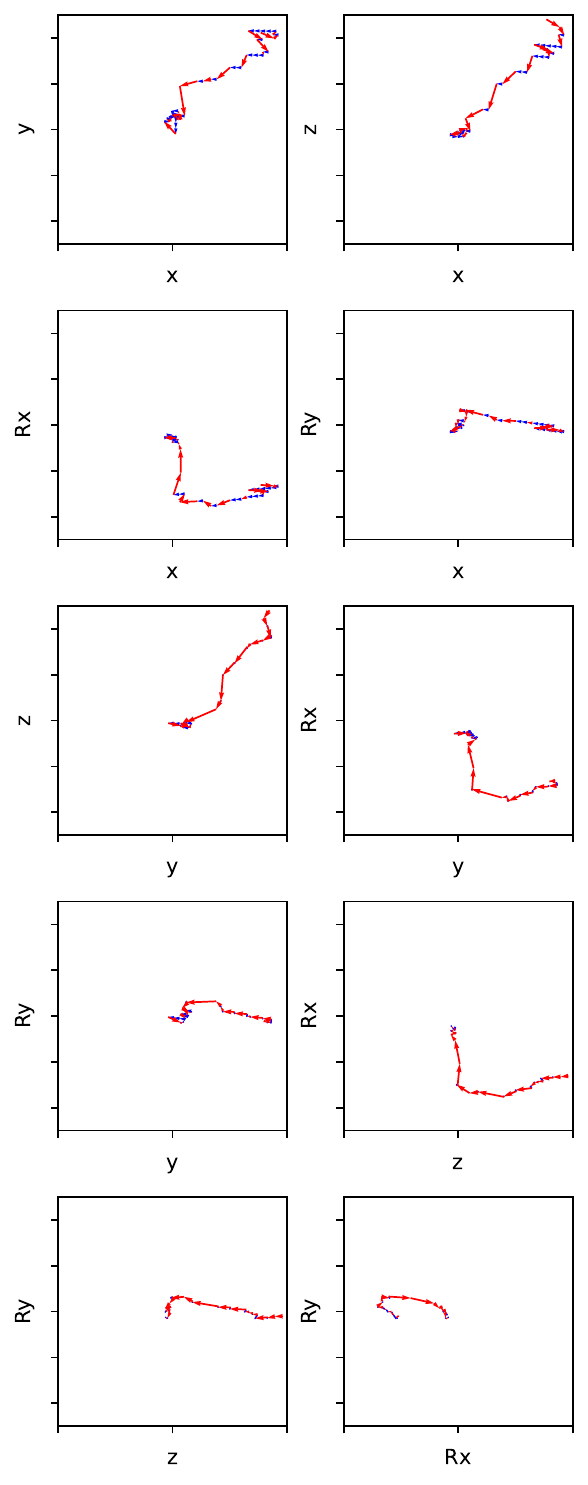}
    }
    \hfill
    \subfloat[]{
    \includegraphics[width=0.27\textwidth]{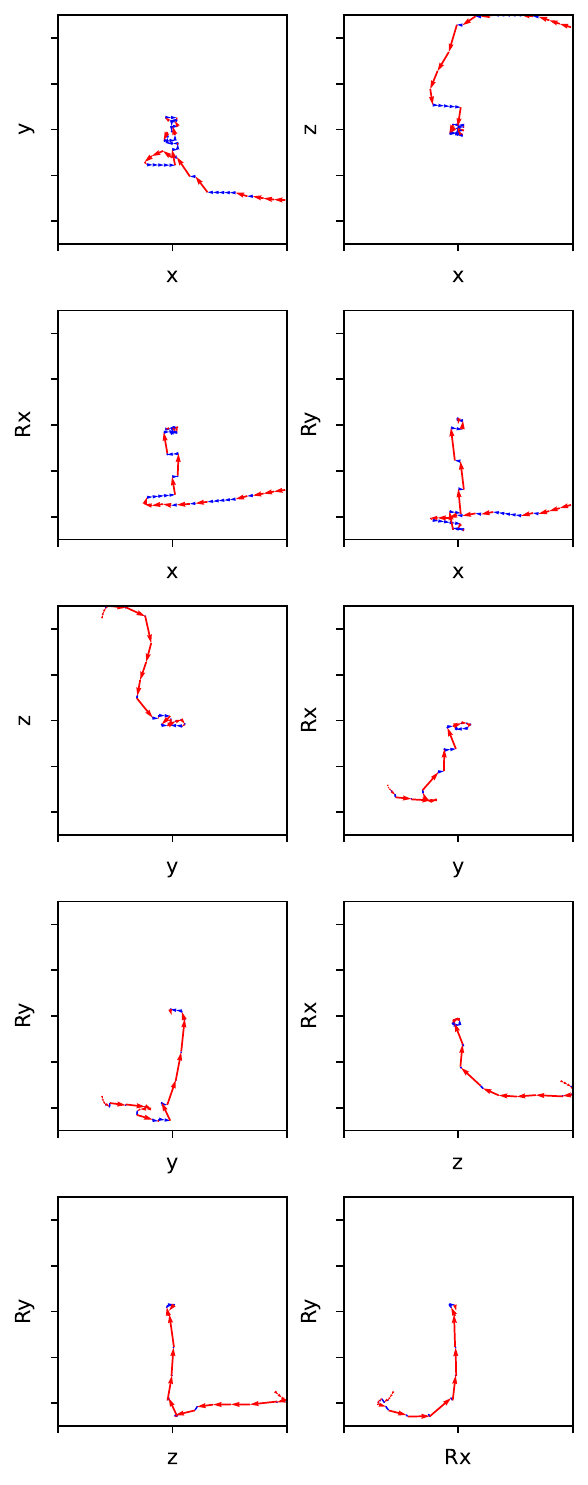}
    }
    \caption{Augmented trajectories generated by \ourmethod{} for $\mathcal{O}_{\mathrm{LP}}$ in $5$
        dimensional hidden position space: Actions coming from the augmentor in
    red and actions from the logging policy in blue.}
\end{figure}


\end{document}